\documentclass[review,onefignum,onetabnum]{siamart171218}

\usepackage{lipsum}
\usepackage{amsfonts}
\usepackage{graphicx}
\usepackage{epstopdf}
\usepackage{algorithm} 
\usepackage{algpseudocode} 
\usepackage{enumitem}
\usepackage{footmisc}
\usepackage{tikz,tikz-3dplot}
\usepackage{pgfplots}
\pgfplotsset{compat=1.11}
\usepgfplotslibrary{fillbetween}
\usetikzlibrary{patterns}
\usetikzlibrary{arrows}
\usetikzlibrary{calc,fit}
\usetikzlibrary{shapes.misc}
\usepackage{color}
\usepackage[normalem]{ulem}

\ifpdf
\DeclareGraphicsExtensions{.eps,.pdf,.png,.jpg}
\else
\DeclareGraphicsExtensions{.eps}
\fi

\newsiamremark{remark}{Remark}
\newsiamremark{hypothesis}{Hypothesis}
\crefname{hypothesis}{Hypothesis}{Hypotheses}
\newsiamthm{claim}{Claim}
\newsiamthm{question}{Question}
\newsiamthm{problem}{Problem}

\headers{Sparse Methods for Subspace-preserving Recovery}{D. P. Robinson,  R. Vidal, and C. You}

\title{Basis Pursuit and Orthogonal Matching Pursuit for Subspace-preserving Recovery: Theoretical Analysis\thanks{Submitted to the editors 12/30/2019. Part of the results have appeared in \cite{You:ICML15,You:CVPR16-SSCOMP}. 
		\funding{This work was funded by the National Science Foundation under contracts 1447822, 1704458, 1736448, and 1934931.}}}

\author{
	Daniel P. Robinson\thanks{Department of Industrial and Systems Engineering, Lehigh University, Bethlehem, PA, USA (\email{daniel.p.robinson@gmail.com}).}
	\and Ren\'e Vidal\thanks{Mathematical Institute for Data Science, Department of Biomedical Engineering, The Johns Hopkins University, Baltimore, MD, USA (\email{rvidal@jhu.edu}).}
	\and Chong You\thanks{Department of Electrical Engineering and Computer Sciences, University of California at Berkeley, Berkeley, CA, USA
		(\email{chong.you1987@gmail.com}). The vast majority of the work was done while this author was at The Johns Hopkins University.}
}

\usepackage{amsopn}
\usepackage{subfig}

\def\0{\mathbf{0}}
\def\1{\mathbf{1}}
\def\a{\mathbf{a}}
\def\c{\mathbf{c}}

\def\b{\mathbf{b}}
\def\h{\mathbf{h}}
\def\p{\mathbf{p}}
\def\q{\mathbf{q}}
\def\u{\mathbf{u}}
\def\v{\mathbf{v}}
\def\w{\mathbf{w}}

\def\A{\mathbf{A}}

\def\Q{\mathbf{Q}}

\def\cA{\mathcal{A}}

\def\cD{\mathcal{D}}

\def\cP{\mathcal{P}}

\def\cK{\mathcal{K}}

\def\cR{\mathcal{R}}
\def\cS{\mathcal{S}}

\def\cV{\mathcal{V}}
\def\cW{\mathcal{W}}

\def\transpose{\top} 

\def\st{\hspace{1em} \mathrm{s.t.} \hspace{0.5em}}
\def\Re{\mathbb{R}}
\def\Sp{\mathbb{S}}

\tikzset{cross/.style={cross out, draw=black, minimum size=2*(#1-\pgflinewidth), inner sep=0pt, outer sep=0pt},
	cross/.default={3pt}}

\tikzset{
	box/.style={rectangle, rounded corners=6pt,
		minimum width=50pt, minimum height=20pt, inner sep=6pt,
		draw=black,thick, fill=white}
}

\tikzset{viewport/.style 2 args={
		x={({cos(-#1)*1cm},{sin(-#1)*sin(#2)*1cm})},
		y={({-sin(-#1)*1cm},{cos(-#1)*sin(#2)*1cm})},
		z={(0,{cos(#2)*1cm})}
}}

\pgfplotsset{only foreground/.style={
		restrict expr to domain={rawx*\CameraX + rawy*\CameraY + rawz*\CameraZ}{-0.05:100},
}}
\pgfplotsset{only background/.style={
		restrict expr to domain={rawx*\CameraX + rawy*\CameraY + rawz*\CameraZ}{-100:0.05}
}}

\newcommand{\ViewAzimuth}{10}
\newcommand{\ViewElevation}{35}
\pgfmathsetmacro{\CameraX}{sin(\ViewAzimuth)*cos(\ViewElevation)}
\pgfmathsetmacro{\CameraY}{-cos(\ViewAzimuth)*cos(\ViewElevation)}
\pgfmathsetmacro{\CameraZ}{sin(\ViewElevation)}

\newcommand{\myparagraph}[1]{\smallskip\noindent\textbf{#1.}}

\definecolor{tiffanyblue}{rgb}{0.04, 0.73, 0.71}
\definecolor{cadmiumgreen}{rgb}{0.0, 0.42, 0.24}
\definecolor{caribbeangreen}{rgb}{0.0, 0.8, 0.6}
\definecolor{amber}{rgb}{1.0, 0.49, 0.0}
\colorlet{pPntColor}{blue} 
\colorlet{dPntColor}{red} 
\colorlet{oPntColor}{orange} 
\colorlet{iPntColor}{caribbeangreen} 
\colorlet{cColor}{orange}
\colorlet{pRgnColor}{yellow} 
\colorlet{dRgnColor}{yellow} 

\DeclareMathOperator*{\card}{card}
\DeclareMathOperator*{\conv}{conv}

\DeclareMathOperator*{\ns}{null}
\DeclareMathOperator*{\range}{range}

\DeclareMathOperator*{\sgn}{sgn}

\DeclareMathOperator*{\spann}{span}

\DeclareMathOperator*{\argmin}{arg\,min}
\DeclareMathOperator*{\Argmin}{Arg\,min}
\DeclareMathOperator*{\argmax}{arg\,max}
\DeclareMathOperator*{\Argmax}{Arg\,max}
\DeclareMathOperator{\BP}{BP}
\DeclareMathOperator{\Dual}{dual}
\DeclareMathOperator{\OMP}{OMP}

\nolinenumbers

\ifpdf
\hypersetup{
	pdftitle={Basis Pursuit and Orthogonal Matching Pursuit for Subspace-preserving Recovery: Theoretical Analysis},
	pdfauthor={D. P. Robinson, R. Vidal and C. You}
}
\fi

\begin{document}
	
	\maketitle
	
	\begin{abstract}
		Given an overcomplete dictionary $\A$ and a signal 
		$\b = \A\c^*$ for some sparse vector $\c^*$ whose nonzero entries correspond to \emph{linearly independent} columns of $\A$,    
		classical sparse signal recovery theory considers the problem of whether $\c^*$ can be recovered 
		as the unique sparsest solution to $\b = \A \c$. 
		It is now well-understood that such recovery is possible by practical algorithms when the dictionary $\A$ is incoherent or restricted isometric. 
		In this paper, we consider the more general case where $\b$ lies in a 
		subspace $\cS_0$ spanned by a subset of \emph{linearly dependent} columns of $\A$, and   the remaining columns are outside of the subspace. 
		In this case, the sparsest representation may not be unique, and the dictionary may not be incoherent or restricted isometric. 
		The goal is to have the representation $\c$ correctly identify the subspace, i.e. the nonzero entries of $\c$ should correspond to columns of $\A$ that are in the subspace $\cS_0$. 
		Such a representation $\c$ is called \textit{subspace-preserving}, a key concept that has found important applications for learning low-dimensional structures in high-dimensional data.
		We present various geometric conditions that guarantee subspace-preserving recovery. 
		Among them, the major results are characterized by the covering radius and the angular distance, which capture the distribution of points in the subspace and the similarity between points in the subspace and points outside the subspace, respectively. 
		Importantly, these conditions do not require the dictionary to be incoherent or restricted isometric. 
		By establishing that the subspace-preserving recovery problem and the classical sparse signal recovery problem are equivalent under common assumptions on the latter, we show that several of our proposed conditions are generalizations of some well-known conditions in the sparse signal recovery literature.
	\end{abstract}
	
	
	
	\section{Introduction}
	Sparse representation is a popular mathematical tool for many applications in areas such as geophysics \cite{Taylor:Geo79}, statistics \cite{Tibshirani:RSS96}, signal processing \cite{Chen:SIAM98} and computer vision \cite{Mairal:FT14}.
	Given a matrix $\A \in \Re^{D \times N}$ called the dictionary and a vector $\b \in \Re^D$, the problem of sparse recovery is to find a vector $\c \in \Re^N$ with the fewest nonzero entries that solves the under-determined linear system $\b = \A\c$. 
	In the field of \emph{compressive sensing}, tools from signal processing, numerical optimization, statistics and applied mathematics have significantly advanced the theoretical understanding of the sparse recovery problem \cite{Candes:SPM08,Elad:SIAM09,Foucart:13}.
	The key insight is that it is often  possible to recover an $N$-dimensional sparse signal $\c$ from $D$ ($D \ll N$) linear measurements $\b = \A \c$ (i.e., the row vectors of $\A$ are viewed as the means to obtain linear measurements).
	It is now well-known that such a $\c$ is the unique sparsest solution to $\b = \A \c$ if $\A$ satisfies the \emph{spark} condition  \cite{DonohoElad:PNAS03}. 
	Moreover, the reconstruction of such a sparse vector $\c$ is not only feasible in theory but also computationally efficient via numerical algorithms such as Basis Pursuit (BP) \cite{Chen:SIAM98} and Orthogonal Matching Pursuit (OMP) \cite{Pati:Asilomar93}. 
	Specifically, such algorithms are guaranteed to recover $\c$ if the dictionary $\A$ satisfies the \emph{incoherence} \cite{Tropp:TIT04,Donoho:TIT06} or \emph{restricted isometry} \cite{Candes:TIT05,Candes:CRAS08,Davenport:TIT10,Mo:ACHA11,Mo:TIT12-RIPinOMP,Cai:TIT14} properties.

	The mathematical theory for sparse recovery triggered the design of new techniques for broader applications \cite{Qaisar:JCN13}. 
	One of the most successful instances is the task of learning {parsimonious representations of} high-dimensional data such as images, videos, audios and genetics \cite{Wright:IEEEProc10,Mairal:FT14}. 
	For such data, the intrinsic dimension is often much smaller than the dimension of the ambient space. 
	For example, while face images taken by a mega-pixel camera are million-dimensional, images of a particular face exhibit low-dimensional structure since their variations can be described by a few factors such as the pose of the head, the expression of the face and the lighting conditions.
	To see how low-dimensionality is connected to sparse recovery, let $\A$ be a data matrix with each column corresponding to a data point, and assume that a subset of the columns is drawn from a low-dimensional linear subspace. 
	Then, for any data point $\b$ drawn from this subspace, there exist sparse solutions to the linear system $\b = \A \c$ whose nonzero entries correspond to columns of $\A$ that lie in the same subspace as $\b$; we say that such a sparse solution $\c$ is \emph{subspace-preserving}. 
	Finding subspace-preserving solutions plays a central role in a range of learning tasks 
	such as  classification \cite{Wright:PAMI09,Elhamifar:TSP12}, clustering \cite{Elhamifar:CVPR09,Dyer:JMLR13,You:CVPR16-SSCOMP,Yang:ECCV16,Tschannen:TIT18}, outlier detection \cite{Soltanolkotabi:AS12,You:CVPR17}, and subset selection \cite{Elhamifar:CVPR12,You:ECCV18}.
	
	As a concrete example, consider the task of face recognition where the goal is to classify a given face image to its correct group by using a collection of face images as a training set.  
	For each face, training images are collected under various illumination conditions. 
	Each image for each face is represented as a column vector that contains its pixel intensity values, and this collection of vectors becomes the columns of the dictionary $\A$. 
	Since the images of a given face under various illumination conditions lie approximately in a low-dimensional linear subspace \cite{Basri:PAMI03}, the columns of  $\A$ approximately lie in a union of subspaces with each subspace corresponding to one of the faces.
	In the test phase, a test image $\b$ is expressed as a sparse linear combination of the training images, i.e. $\b = \A \c$. 
	Owing to the subspace structure, the representation $\c$ is expected to be subspace-preserving, whereby the label of $\b$ is given by the label of the training images that correspond to nonzero entries in $\c$. This method for face recognition is known as the sparse representation-based classification (SRC) \cite{Wright:PAMI09}.
	
	
	The \emph{subspace-preserving recovery} problem (the focus of this paper) needs to be clearly distinguished from the \emph{sparse signal recovery} problem extensively studied in the literature \cite{Candes:SPM08,Elad:SIAM09,Foucart:13}. 
	Although both problems concern solving a sparse recovery problem, their objectives and standard for correctness are different. 
	In sparse signal recovery, the goal is to recover a sparse signal $\c$ from the measurements $\b = \A \c$, and the computation is carried out under the presumption that $\c$ is the unique sparsest solution.
	Consequently, the uniqueness of the sparsest solution to $\b = \A \c$ is a fundamental property in the theoretical study of sparse signal recovery. 
	This is in sharp contrast to subspace-preserving recovery for which the uniqueness generally does not hold.
	Instead of aiming to recover a unique sparsest solution, subspace-preserving recovery is regarded as successful when the nonzero entries in the recovered solution correspond to columns of $\A$ that lie in the same subspace as $\b$.
	This sets the stage for the development of radically new theory for subspace-preserving recovery.

	\subsection{Motivation for the proposed research}
	
	It is surprising that sparse recovery methods such as SRC work well for subspace structured data since the spark, coherence, and restricted isometry conditions are often violated.
	For example, if the training set contains two images of the same face taken under similar illumination conditions, then the coherence of the dictionary is close to one. 
	The gap between the existing theory and the empirical success of sparse recovery for subspace structured data calls for the study of new theoretical guarantees for subspace-preserving recovery. 
	

	\subsection{Problems studied}
	\label{sec:problem_formulation}
	
	Let $\cA := \{ \a_j \}_{j=1}^N \subset \Re^D $ be a finite dictionary composed of a subset of points $\cA_0 \subseteq \cA$ that span a $d_0$-dimensional subspace $\cS_0$\footnote{This paper focuses on the case where the subspace $\cS$ is \emph{linear} (i.e., passes through the origin). We refer the reader to \cite{Li:JSTSP18,You:ICCV19} for an account of the case where $\cS$ is \emph{affine}.}, with the remaining points $\cA_- := \cA \setminus \cA_0$ being outside the subspace $\cS_0$. 
	We use $\A$ to denote the matrix that contains all points from $\cA$ as its columns, and likewise for $\A_0$ and $\A_-$. We assume throughout that all vectors in $\cA$ have unit $\ell_2$ norm.
	For any $\0 \neq \b\in\cS_0$, we formally define a subspace-preserving representation as follows.
	\begin{definition}[subspace-preserving representation]\label{def:subspace-preserving}
		Given a dictionary $\cA$, matrix $\A$, and vector $\b$ as above, a vector $\c$ satisfying  $\b = \A \c$ is called a \emph{subspace-preserving representation} if $c_j \ne 0$ implies that $\a_j \in \cA_0$ for all $j\in\{1,\dots, N\}$. 
	\end{definition}
	
	In words, a $\c$ satisfying $\b = \A \c$ is subspace-preserving if the points in $\cA$ that correspond to the nonzero entries of $\c$ lie in $\cS_0$.
	For the face recognition application described earlier, we can think of $\cA_0$ as the set of training images for a particular face that spans the subspace $\cS_0$, and $\cA_-$ as the set of training data from all other faces. 
	Then, roughly speaking, determining whether the vector $\b$ belongs to the subspace $\cS_0$ is equivalent to deciding whether a subspace-preserving representation of $\b$ exists. 
	
	An important observation for recovering a subspace-preserving representation is that there always exist such representations with at most $d_0$ nonzero entries.
	Indeed, one can always find a set of $d_0$ linearly independent data points from $\cA_0$ that spans $\cS_0$, and can thus linearly represent all points $\b \in \cS_0$.
	When $d_0$ is small relative to $N$, such representations are sparse.
	This suggests that, among all representation vectors $\c$ that satisfy $\b = \A \c$, the sparsest ones should be subspace-preserving.
	
	In this paper, we provide a theoretical justification for such ideas. 
	Since the problem of finding the sparsest solution is NP-hard in general \cite{Natarajan:SIAM95}, we focus our theoretical analysis on sparse recovery using algorithms BP and OMP (see \cref{sec:BPOMP} for an overview of these methods). 
	We consider the following theoretical question:
	What conditions on $\A$ and $\b$ ensure that BP and OMP produce subspace-preserving solutions?
	In answering this question, we distinguish two different but related problems. 
	
	
	\begin{problem}[instance recovery problem]
		Given $\A$ and $\0 \neq \b\in\cS_0$ as above, under what conditions is subspace-preserving recovery guaranteed by BP and OMP?
	\end{problem}
	\begin{problem}[universal recovery problem]
		Given $\A$ as above, under what conditions is subspace-preserving recovery guaranteed by BP and OMP for all $\0 \neq \b\in\cS_0$?
	\end{problem}
	
	We derive instance and universal recovery conditions in \cref{sec:instance_recovery} and \cref{sec:universal_recovery}. 
	
	\begin{figure*}
		\small
		\centering
		\begin{tikzpicture}
		\node[box, text width=6.1cm, align=center] (1) at(0, 2.6) 
		{\text{$\forall\c \in \BP(\A, \b)$ are subspace-preserving}};
		\node[box, text width=6.1cm, align=center] (2) at(0, 1.2)
		{Tight \\
			Instance Dual Condition (T-IDC): \\$\exists \v \in \BP_{\Dual}(\A_0, \b)$ s.t.~$\| \A_-^\transpose \v \|_\infty < 1$};
		\node[box, text width=6.1cm, align=center] (7) at(0, -0.4)
		{Instance Dual Condition (IDC): \\
			\text{$\exists \v \in \cD(\A_0, \b)$ s.t.~$\| \A_-^\transpose \v \|_\infty < 1$}};
		\node[box, text width=6.1cm, align=center] (3) at(0, -2.0)
		{Geometric \\
			Instance Dual Condition (G-IDC): \\
			\text{$\exists \v \in \cD(\A_0, \b)$ s.t. $\gamma_0 > \theta(\{\pm\v\}, \cA_-)$}};
		\node[box, text width=5.8cm, align=center] (4) at(6.6, 2.6) 
		{\text{$\OMP(\A, \b)$ is subspace-preserving}};
		\node[box, text width=5.8cm, align=center] (5) at(6.6, -0.4) 
		{Instance Residual Condition (IRC): \\
			\text{$\forall \v \in \cR(\A_0, \b),\|\A_-^\transpose \v\|_\infty < \|\A_0^\transpose \v\|_\infty$}};
		\node[box, text width=5.8cm, align=center] (6) at(6.6, -2.0)
		{Geometric \\
			Instance Residual Condition (G-IRC): \\
			\text{$\forall \v \in \cR(\A_0, \b)$, $\gamma_0 > \theta(\{\pm\v\}, \cA_-)$}};
		\draw[implies-implies, double equal sign distance] (2) -- (1);
		\draw[       -implies, double equal sign distance] (3) -- (7);
		\draw[       -implies, double equal sign distance] (7) -- (2);		
		\draw[       -implies, double equal sign distance] (5) -- (4);
		\draw[       -implies, double equal sign distance] (6) -- (5);
		\end{tikzpicture}
		\caption{Instance recovery conditions for BP and OMP.
			The first row states instance subspace-preserving recovery by BP and OMP. 
			The second row has a tight condition for ensuring subspace-preserving recovery for BP characterized by the set of dual solutions $\BP_{\Dual}(\A_0, \b)$ (see \cref{thm:T-IDC}). 
			The third row gives the dual and residual conditions characterized by the dual points $\cD(\A_0, \b)$ (see \cref{thm:IDC}) and residual points $\cR(\A_0, \b)$ (see \cref{thm:IRC}), respectively. 
			The fourth row gives geometric conditions characterized by the covering radius $\gamma_0$ and the angular distance between $\cA_-$ and either (i) an arbitrary point from the set of dual points $\cD(\A_0, \b)$ for BP (see \cref{thm:G-IDC}), or (ii) all vectors from the set of residual points $\cR(\A_0, \b)$ for OMP (see \cref{thm:G-IRC}).}
		\label{fig:instance-recovery-conditions}
	\end{figure*}

	\begin{figure*}[h!]
		\small
		\centering
		\begin{tikzpicture}
		\node[box, text width=6.5cm, align=center] (1) at (0, 5.25)
		{\text{$\forall \b \in \cS_0\setminus \{\0\}$} \\
			\text{$\forall \c \in \BP(\A, \b)$ are subspace-preserving}};
		\node[box, text width=6.5cm, align=center] (3) at (0, 3.6)
		{Tight \\
			Universal Dual Condition (T-UDC): \\
			\text{$\forall \w \in \cD(\A_0), ~ \exists \v \in \w + \cS_0^\perp ~\text{s.t.}~ \|\A_-^\transpose \v\|_\infty < 1$}};
		\node[box, text width=6.25cm, align=center] (4) at (0, 1.95)
		{Universal Dual Condition (UDC): \\
			\text{$\forall \v \in \cD(\A_0)$, $\|\A_-^\transpose \v\|_\infty < 1$}};
		\node[box, text width=7.5cm, align=center] (5) at (3.0, 0.3)
		{Geometric  \\
			Universal Dual Condition (G-UDC): \\
			\text{$\forall \v \in \cD(\A_0)$, $\gamma_0 < \theta(\{\pm\v\}, \cA_-)$}};
		\node[box, text width=7.5cm, align=center] (6) at (3.0, -1.5) 
		{Geometric \\
			Universal Subspace Condition (G-USC): \\
			\text{$\forall \v \in \cS_0 \setminus \{\0\}$, $\gamma_0 < \theta(\{\pm\v\}, \cA_-)$}};
		\node[box, text width=5.4cm, align=center] (2) at (6.6, 5.25)
		{\text{$\forall \b \in \cS_0\setminus \{\0\}$} \\ 
			\text{$\OMP(\A, \b)$ is subspace-preserving}};
		\node[box, text width=5.4cm, align=center] (7) at (6.6, 1.95)
		{Universal Residual Condition (URC): \\
			\text{$\forall \v \in \cS_0 \setminus \{\0\}$, $\|\A_-^\transpose \v\|_\infty < \|\A_0^\transpose \v\|_\infty$}};
		\draw[implies-implies, double equal sign distance] (3) -- (1);
		\draw[       -implies, double equal sign distance] (4) -- (3);
		\draw[       -implies, double equal sign distance] (7) -- (2);
		\draw[implies-implies, double equal sign distance] (4) -- (7);
		\draw[       -implies, double equal sign distance] (5) -- (4);
		\draw[       -implies, double equal sign distance] (5) -- (7);		
		\draw[       -implies, double equal sign distance] (6) -- (5);
		\end{tikzpicture}
		\caption{Universal recovery conditions for BP and OMP. 
			The first row states universal subspace-preserving recovery by BP and OMP.
			The second row has a tight condition for ensuring universal subspace-preserving recovery for BP characterized by the set of dual points $\cD(\A_0)$ (see \cref{thm:T-UDC}).  
			The third row gives the dual and residual conditions characterized by the dual points $\cD(\A_0)$ (see \cref{thm:UDC}) and residual points (see \cref{thm:URC}).
			Geometric conditions characterized by the inradius $r_0$ and the coherence between $\cA_-$ and  (i) the set of dual points $\cD(\A_0)$ (see \cref{thm:G-UDC}) and (ii) the subspace $\cS_0$ (see \cref{thm:G-USC}), are given in rows four and five, respectively.
		}
		\label{fig:result-flowchart}
	\end{figure*}

	\subsection{Paper contributions} \label{subsec:summary}
	This paper makes the following contributions towards understanding the conditions for subspace-preserving recovery. 
	
	\begin{itemize}[topsep=0.2em,itemsep=0.2em,leftmargin=*]
		\item \myparagraph{Instance recovery conditions} We present conditions for ensuring instance recovery by both BP and OMP. 
		A summary of the conditions is given in Figure~\ref{fig:instance-recovery-conditions}.
		\begin{itemize}[topsep=0.2em,itemsep=0.2em,leftmargin=*]
			\item {We derive a Tight Instance Recovery Condition (T-IDC) for BP, which requires the existence of a point in $\BP_{\Dual}(\A_0, \b)$, the set of dual solutions to the BP problem with dictionary $\A_0$, that is well separated from points in $\A_{-}$.}
			\item {We derive the Instance Dual Condition (IDC) and the Instance Residual Condition (IRC) for BP and OMP, respectively. These \emph{sufficient conditions} require the set of dual points $\cD(\A_0, \b)$ for BP or the set of residual vectors $\cR(\A_0, \b)$ for OMP with dictionary $\A_0$ to be well separated from points in $\A_{-}$.}
			\item {We derive the Geometric Instance Dual Condition (G-IDC) and the Geometric Instance Residual Condition (G-IRC) for BP and OMP, respectively. These conditions are characterized by the covering radius $\gamma_0$ of points associated with $\cA_0$, and the angular distance between $\cA_-$ and either (i) the set of dual points $\cD(\A_0, \b)$ for BP, 			or (ii) the set of residual points $\cR(\A_0, \b)$ for OMP. }
		\end{itemize}
		\item \myparagraph{Universal recovery conditions}
		{We derive tight, sufficient, and geometric conditions for universal recovery by both BP and OMP (see Figure~\ref{fig:result-flowchart}) that parallel our conditions for instance recovery. The primary difference is that sufficient conditions for universal recovery by BP and OMP are equivalent, and are implied by a shared Geometric Universal Dual Condition (G-UDC) that depends on the covering radius of points associated with $\cA_0$ and the angular distance between the set of dual points $\cD(\A_0)$ and the data outside of the subspace $\cA_{-}$.} 
		%
		%
		\item \myparagraph{Connections to sparse signal recovery} We show that sparse signal recovery is a special case of subspace-preserving recovery. 
		By applying our instance and universal conditions to this specialized setting we obtain {new} conditions for sparse signal recovery.  
		We clarify the connections between these conditions and well-known sparse signal recovery conditions such as the exact recovery condition \cite{Tropp:TIT04}, the null space property \cite{Cohen:JAMS09} and the mutual coherence condition \cite{Donoho:TIT01}.
	\end{itemize}

	The problems of instance recovery by OMP, instance recovery by BP and universal recovery have been studied in the conference proceedings article \cite{You:CVPR16-SSCOMP}, the related work \cite{Soltanolkotabi:AS12}, and another conference proceedings article \cite{You:ICML15}, respectively. 
    This paper integrates part of the existing results developed under different settings and characterized by different mathematical concepts, and provides a unified study in which conditions for both BP and OMP, as well as for both instance and universal recovery, are derived and formulated in comparable forms. 
	This makes it easy to understand their relationship and to interpret their differences.
	In addition, we fill in missing components in existing works by introducing new conditions (e.g., the IDC and the URC), proving tightness of existing conditions (i.e., the T-IDC), and establishing equivalence of our new conditions (e.g., the T-UDC) with those in prior works.   
	This also allows us to identify the close relationship between subspace-preserving and sparse signal recovery conditions.

	\ifdefined\draft
	\else
	
	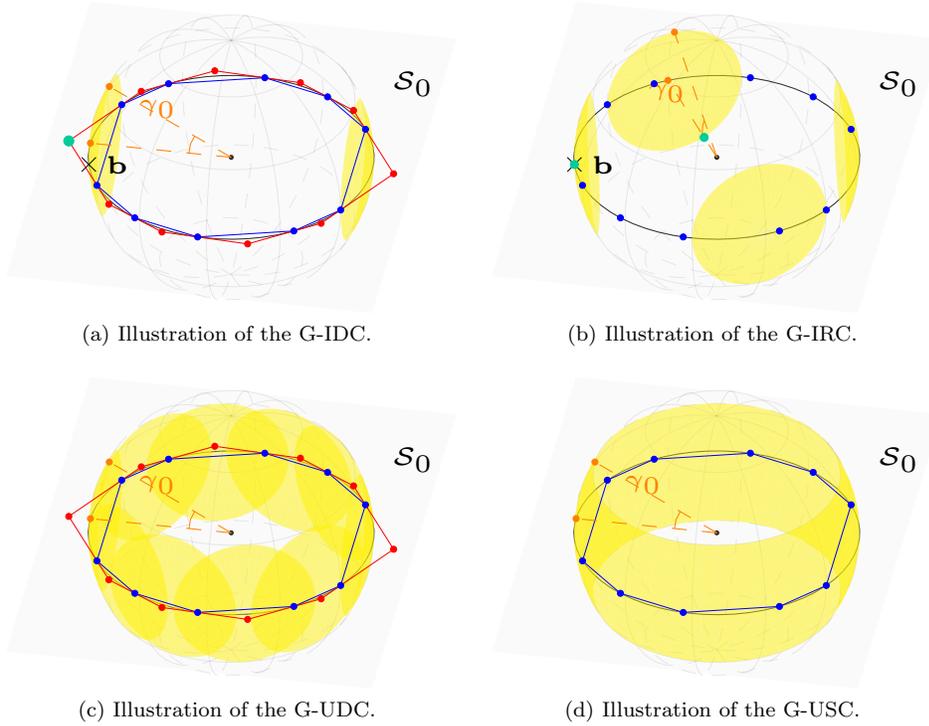
\begin{figure*}[t]
		\centering
		\def\nPoint{5}
		\def\primalAngle
		{{30.00, 57.60, 86.40, 126.00, 150.00, 210.00}}
		\def\dualRadius
		{{1.030, 1.032, 1.063, 1.022, 1.154, 1.030}}
		\def\dualAngle
		{{43.80, 72.00, 106.20, 138.00, 180.00, 223.80}}
		\def\instancePoint{195}
		\def\residualPoint{120}
		\def\instanceDualRadius{1.154}
		\def\instanceDualAngle{180.00}
		\def\instanceResidualAngle{120}
		\def\instanceResidualRadius{0.2588} 
		\subfloat[Illustration of the G-IDC.\label{fig:geometry-G-IDC}]
		{
			\begin{tikzpicture}[scale=1.9]
			\begin{axis}[
			hide axis,
			view={\ViewAzimuth}{\ViewElevation},     
			every axis plot/.style={very thin},
			disabledatascaling,                      
			anchor=origin,                           
			viewport={\ViewAzimuth}{\ViewElevation}, 
			xmin=-3,   xmax=3,
			ymin=-3,   ymax=3,
			]
			
			\def\Gamma{30}
			\draw [black, fill=black] plot [mark=*, mark size=0.3] coordinates{(0, 0, 0)};
			\addplot3 [domain = 0:360, samples = 60, samples y=1]
			({cos(x)}, {sin(x)}, {0});
			\foreach \i in {0, 30,...,150}
			\addplot3 [domain = 0:360, samples = 60, samples y=1, draw opacity = 0.1, only foreground]
			({cos \i*sin(x)}, {sin \i*sin(x)}, {cos(x)});
			\foreach \i in {0, 30,...,150}
			\addplot3 [domain = 0:360, samples = 60, samples y=1, draw opacity = 0.1, only background, dashed]
			({cos \i*sin(x)}, {sin \i*sin(x)}, {cos(x)});
			\foreach \j in {30, 60, ..., 150}
			\addplot3 [domain = 0:360, samples = 60, samples y=1, draw opacity = 0.1, only foreground]
			({cos(x)*sin \j}, {sin(x)*sin \j}, {cos \j});
			\foreach \j in {30, 60, ..., 150}
			\addplot3 [domain = 0:360, samples = 60, samples y=1, draw opacity = 0.1, only background, dashed]
			({cos(x)*sin \j}, {sin(x)*sin \j}, {cos \j});
			\def\Sx{1.3}         
			\def\Sy{1.7}
			\filldraw[draw=none,fill=gray!20, opacity=0.2] 
			(-\Sx,-\Sy,0) -- (-\Sx, \Sy, 0) -- (\Sx, \Sy, 0) -- (\Sx,  -\Sy, 0) -- cycle;
			\node[black] at (1.1, 1.1, 0) {\tiny$\mathcal{S}_0$};
			\addplot3[domain=0:360, samples=60, samples y=5, domain y=0:\Gamma, surf, shader=flat, color=dRgnColor, draw opacity = 0.0, fill opacity=0.5]
			(
			{cos \instanceDualAngle * cos(y) - sin \instanceDualAngle * sin(y) * cos(x)}, 
			{sin \instanceDualAngle * cos(y) + cos \instanceDualAngle * sin(y) * cos(x)}, 
			{sin(y) * sin(x)}                
			);
			\addplot3[domain=0:360, samples=60, samples y=5, domain y=0:\Gamma, surf, shader=flat, color=dRgnColor, draw opacity = 0.0, fill opacity=0.5]
			(
			{- cos \instanceDualAngle * cos(y) + sin \instanceDualAngle * sin(y) * cos(x)}, 
			{-sin \instanceDualAngle * cos(y) - cos \instanceDualAngle * sin(y) * cos(x)}, 
			{sin(y) * sin(x)}                
			);
			\foreach \j in {1, 2, ..., \nPoint}
			{
				\pgfmathsetmacro\angle{\primalAngle[\j-1]}
				\addplot3[only marks,mark=*,mark size=0.6,pPntColor, fill=pPntColor] coordinates { (cos \angle, sin \angle, 0) };
				\addplot3[only marks,mark=*,mark size=0.6,pPntColor, fill=pPntColor] coordinates { (-cos \angle, -sin \angle, 0) };
			}    
			\foreach \j in {1, 2, ..., \nPoint}
			{
				\pgfmathsetmacro\angle{\dualAngle[\j-1]}
				\pgfmathsetmacro\radius{\dualRadius[\j-1]}
				\addplot3[only marks,mark=*,mark size=0.6,dPntColor, fill=dPntColor] coordinates { (\radius * cos \angle, \radius * sin \angle, 0) };
				\addplot3[only marks,mark=*,mark size=0.6,dPntColor, fill=dPntColor] coordinates { (-\radius * cos \angle, -\radius * sin \angle, 0) };
			}
			\foreach \j in {1, 2, ..., \nPoint}
			{
				\pgfmathsetmacro\anglecurr{\primalAngle[\j-1]}
				\pgfmathsetmacro\anglenext{\primalAngle[\j]}
				\addplot3[solid, pPntColor] coordinates { 
					(cos \anglecurr, sin \anglecurr, 0) 
					(cos \anglenext, sin \anglenext, 0)};
				\addplot3[solid, pPntColor] coordinates { 
					(-cos \anglecurr, -sin \anglecurr, 0) 
					(-cos \anglenext, -sin \anglenext, 0)};
			}
			\foreach \j in {1, 2, ..., \nPoint}
			{
				\pgfmathsetmacro\anglecurr{\dualAngle[\j-1]}
				\pgfmathsetmacro\anglenext{\dualAngle[\j]}
				\pgfmathsetmacro\radiuscurr{\dualRadius[\j-1]}
				\pgfmathsetmacro\radiusnext{\dualRadius[\j]}
				
				\addplot3[solid, dPntColor] coordinates {
					(\radiuscurr * cos \anglecurr, \radiuscurr * sin \anglecurr, 0) 
					(\radiusnext * cos \anglenext, \radiusnext * sin \anglenext, 0)}; 
				\addplot3[solid, dPntColor] coordinates {
					(-1 * \radiuscurr * cos \anglecurr, -1 * \radiuscurr * sin \anglecurr, 0) 
					(-1 * \radiusnext * cos \anglenext, -1 * \radiusnext * sin \anglenext, 0)}; 
			}
			\def\gammaAngle{180}
			\addplot3[dashed, cColor] coordinates {
				(0,0,0) 
				(cos \gammaAngle, sin \gammaAngle, 0)};
			\addplot3[dashed, cColor] coordinates {
				(0,0,0)
				(cos \Gamma * cos \gammaAngle, cos \Gamma * sin \gammaAngle, sin \Gamma)};
			\addplot3[only marks,mark=*,mark size=0.6,cColor, fill=cColor] coordinates{(cos \gammaAngle, sin \gammaAngle, 0)};
			\addplot3[only marks,mark=*,mark size=0.6,cColor, fill=cColor] coordinates{(cos \Gamma * cos \gammaAngle, cos \Gamma * sin \gammaAngle, sin \Gamma)};
			\addplot3 [domain = 0:\Gamma, samples = 3, samples y=1, cColor]
			({0.3* cos \gammaAngle * cos(x)}, {0.3 * sin \gammaAngle * cos(x)}, {0.3* sin(x)});
			\node[above left, cColor] at (0.3* cos \gammaAngle * cos \Gamma, 0.3 * sin \gammaAngle * cos \Gamma, 0.3* sin \Gamma) {\tiny$\gamma_0$};
			\addplot3[only marks,mark=x,mark size=2, black, fill=black] coordinates { (cos \instancePoint, sin \instancePoint, 0.0) };
			\node[black, right] at (cos \instancePoint, sin \instancePoint, 0.0)  {\tiny$\b$}; 
			\addplot3[only marks,mark=*,mark size=1, iPntColor, fill=iPntColor] coordinates { (\instanceDualRadius * cos \instanceDualAngle, \instanceDualRadius * sin \instanceDualAngle, 0) };
			\end{axis}
			\end{tikzpicture}
		}
		~~
		\subfloat[Illustration of the G-IRC.\label{fig:geometry-G-IRC}] 
		{
			\begin{tikzpicture}[scale=1.9]
			\begin{axis}[
			hide axis,
			view={\ViewAzimuth}{\ViewElevation},     
			every axis plot/.style={very thin},
			disabledatascaling,                      
			anchor=origin,                           
			viewport={\ViewAzimuth}{\ViewElevation}, 
			xmin=-3,   xmax=3,
			ymin=-3,   ymax=3,
			]
			
			\def\Gamma{30}
			\draw [black, fill=black] plot [mark=*, mark size=0.3] coordinates{(0, 0, 0)};
			\addplot3 [domain = 0:360, samples = 60, samples y=1]
			({cos(x)}, {sin(x)}, {0});
			\foreach \i in {0, 30,...,150}
			\addplot3 [domain = 0:360, samples = 60, samples y=1, draw opacity = 0.1, only foreground]
			({cos \i*sin(x)}, {sin \i*sin(x)}, {cos(x)});
			\foreach \i in {0, 30,...,150}
			\addplot3 [domain = 0:360, samples = 60, samples y=1, draw opacity = 0.1, only background, dashed]
			({cos \i*sin(x)}, {sin \i*sin(x)}, {cos(x)});
			\foreach \j in {30, 60, ..., 150}
			\addplot3 [domain = 0:360, samples = 60, samples y=1, draw opacity = 0.1, only foreground]
			({cos(x)*sin \j}, {sin(x)*sin \j}, {cos \j});
			\foreach \j in {30, 60, ..., 150}
			\addplot3 [domain = 0:360, samples = 60, samples y=1, draw opacity = 0.1, only background, dashed]
			({cos(x)*sin \j}, {sin(x)*sin \j}, {cos \j});
			\def\Sx{1.3}         
			\def\Sy{1.7}
			\filldraw[draw=none,fill=gray!20, opacity=0.2] 
			(-\Sx,-\Sy,0) -- (-\Sx, \Sy, 0) -- (\Sx, \Sy, 0) -- (\Sx,  -\Sy, 0) -- cycle;
			\node[black] at (1.1, 1.1, 0) {\tiny$\mathcal{S}_0$};
			\foreach \j in {1, 2, ..., \nPoint}
			{
				\pgfmathsetmacro\angle{\primalAngle[\j-1]}
				\addplot3[only marks,mark=*,mark size=0.6,pPntColor, fill=pPntColor] coordinates { (cos \angle, sin \angle, 0) };
				\addplot3[only marks,mark=*,mark size=0.6,pPntColor, fill=pPntColor] coordinates { (-cos \angle, -sin \angle, 0) };
			}   
			\addplot3[domain=0:360, samples=60, samples y=5, domain y=0:\Gamma, surf, shader=flat, color=dRgnColor, draw opacity = 0.0, fill opacity=0.5]
			(
			{cos \instancePoint * cos(y) - sin \instancePoint * sin(y) * cos(x)}, 
			{sin \instancePoint * cos(y) + cos \instancePoint * sin(y) * cos(x)}, 
			{sin(y) * sin(x)}                
			);
			\addplot3[domain=0:360, samples=60, samples y=5, domain y=0:\Gamma, surf, shader=flat, color=dRgnColor, draw opacity = 0.0, fill opacity=0.5]
			(
			{- cos \instancePoint * cos(y) + sin \instancePoint * sin(y) * cos(x)}, 
			{-sin \instancePoint * cos(y) - cos \instancePoint * sin(y) * cos(x)}, 
			{sin(y) * sin(x)}                
			); 
			
			\addplot3[domain=0:360, samples=60, samples y=5, domain y=0:\Gamma, surf, shader=flat, color=dRgnColor, draw opacity = 0.0, fill opacity=0.5]
			(
			{cos \residualPoint* cos(y) - sin \residualPoint * sin(y) * cos(x)}, 
			{sin \residualPoint * cos(y) + cos \residualPoint * sin(y) * cos(x)}, 
			{sin(y) * sin(x)}                
			);
			\addplot3[domain=0:360, samples=60, samples y=5, domain y=0:\Gamma, surf, shader=flat, color=dRgnColor, draw opacity = 0.0, fill opacity=0.5]
			(
			{- cos \residualPoint * cos(y) + sin \residualPoint * sin(y) * cos(x)}, 
			{-sin \residualPoint * cos(y) - cos \residualPoint * sin(y) * cos(x)}, 
			{sin(y) * sin(x)}                
			); 
			\addplot3[dashed, cColor] coordinates {
				(0,0,0) 
				(cos \residualPoint, sin \residualPoint, 0)};
			\addplot3[dashed, cColor] coordinates {
				(0,0,0)
				(cos \Gamma * cos \residualPoint, cos \Gamma * sin \residualPoint, sin \Gamma)};
			\addplot3[only marks,mark=*,mark size=0.6,cColor, fill=cColor] coordinates{(cos \residualPoint, sin \residualPoint, 0)};
			\addplot3[only marks,mark=*,mark size=0.6,cColor, fill=cColor] coordinates{(cos \Gamma * cos \residualPoint, cos \Gamma * sin \residualPoint, sin \Gamma)};
			\addplot3 [domain = 0:\Gamma, samples = 3, samples y=1, cColor]
			({0.3* cos \residualPoint * cos(x)}, {0.3 * sin \residualPoint * cos(x)}, {0.3* sin(x)});
			\node[above left, cColor] at (0.3* cos \residualPoint * cos \Gamma, 0.3 * sin \residualPoint * cos \Gamma, 0.3* sin \Gamma) {\tiny$\gamma_0$};
			\addplot3[only marks,mark=x,mark size=2, black, fill=black] coordinates { (cos \instancePoint, sin \instancePoint, 0.0) };
			\node[black, right] at (cos \instancePoint, sin \instancePoint, 0.0)  {\tiny$\b$}; 
			\addplot3[only marks,mark=*,mark size=0.8, iPntColor, fill=iPntColor] coordinates { (cos \instancePoint, sin \instancePoint, 0) };
			\addplot3[only marks,mark=*,mark size=0.8, iPntColor, fill=iPntColor] coordinates { 			(\instanceResidualRadius * cos \instanceResidualAngle, \instanceResidualRadius * sin \instanceResidualAngle, 0) };
			\end{axis}
			\end{tikzpicture}
		}
		\\
		\subfloat[Illustration of the G-UDC.\label{fig:geometry-3D-DRC}] 
		{
			\begin{tikzpicture}[scale=1.9]
			\begin{axis}[
			hide axis,
			view={\ViewAzimuth}{\ViewElevation},     
			every axis plot/.style={very thin},
			disabledatascaling,                      
			anchor=origin,                           
			viewport={\ViewAzimuth}{\ViewElevation}, 
			xmin=-3,   xmax=3,
			ymin=-3,   ymax=3,
			]
			
			\def\Gamma{30}
			\draw [black, fill=black] plot [mark=*, mark size=0.3] coordinates{(0, 0, 0)};
			\addplot3 [domain = 0:360, samples = 60, samples y=1]
			({cos(x)}, {sin(x)}, {0});
			\foreach \i in {0, 30,...,150}
			\addplot3 [domain = 0:360, samples = 60, samples y=1, draw opacity = 0.1, only foreground]
			({cos \i*sin(x)}, {sin \i*sin(x)}, {cos(x)});
			\foreach \i in {0, 30,...,150}
			\addplot3 [domain = 0:360, samples = 60, samples y=1, draw opacity = 0.1, only background, dashed]
			({cos \i*sin(x)}, {sin \i*sin(x)}, {cos(x)});
			\foreach \j in {30, 60, ..., 150}
			\addplot3 [domain = 0:360, samples = 60, samples y=1, draw opacity = 0.1, only foreground]
			({cos(x)*sin \j}, {sin(x)*sin \j}, {cos \j});
			\foreach \j in {30, 60, ..., 150}
			\addplot3 [domain = 0:360, samples = 60, samples y=1, draw opacity = 0.1, only background, dashed]
			({cos(x)*sin \j}, {sin(x)*sin \j}, {cos \j});
			\def\Sx{1.3}         
			\def\Sy{1.7}
			\filldraw[draw=none,fill=gray!20, opacity=0.2] 
			(-\Sx,-\Sy,0) -- (-\Sx, \Sy, 0) -- (\Sx, \Sy, 0) -- (\Sx,  -\Sy, 0) -- cycle;
			\node[black] at (1.1, 1.1, 0) {\tiny$\mathcal{S}_0$};
			\foreach \j in {1, 2, ..., \nPoint}
			{
				\pgfmathsetmacro\angle{\dualAngle[\j-1]}
				\addplot3[domain=0:360, samples=60, samples y=5, domain y=0:\Gamma, surf, shader=flat, color=dRgnColor, draw opacity = 0.0, fill opacity=0.5]
				(
				{cos \angle * cos(y) - sin \angle * sin(y) * cos(x)}, 
				{sin \angle * cos(y) + cos \angle * sin(y) * cos(x)}, 
				{sin(y) * sin(x)}                
				);
				\addplot3[domain=0:360, samples=60, samples y=5, domain y=0:\Gamma, surf, shader=flat, color=dRgnColor, draw opacity = 0.0, fill opacity=0.5]
				(
				{- cos \angle * cos(y) + sin \angle * sin(y) * cos(x)}, 
				{-sin \angle * cos(y) - cos \angle * sin(y) * cos(x)}, 
				{sin(y) * sin(x)}                
				);
			}
			\foreach \j in {1, 2, ..., \nPoint}
			{
				\pgfmathsetmacro\angle{\primalAngle[\j-1]}
				\addplot3[only marks,mark=*,mark size=0.6,pPntColor, fill=pPntColor] coordinates { (cos \angle, sin \angle, 0) };
				\addplot3[only marks,mark=*,mark size=0.6,pPntColor, fill=pPntColor] coordinates { (-cos \angle, -sin \angle, 0) };
			}     
			\foreach \j in {1, 2, ..., \nPoint}
			{
				\pgfmathsetmacro\angle{\dualAngle[\j-1]}
				\pgfmathsetmacro\radius{\dualRadius[\j-1]}
				\addplot3[only marks,mark=*,mark size=0.6,dPntColor, fill=dPntColor] coordinates { (\radius * cos \angle, \radius * sin \angle, 0) };
				\addplot3[only marks,mark=*,mark size=0.6,dPntColor, fill=dPntColor] coordinates { (-\radius * cos \angle, -\radius * sin \angle, 0) };
			}
			\foreach \j in {1, 2, ..., \nPoint}
			{
				\pgfmathsetmacro\anglecurr{\primalAngle[\j-1]}
				\pgfmathsetmacro\anglenext{\primalAngle[\j]}
				\addplot3[solid, pPntColor] coordinates { 
					(cos \anglecurr, sin \anglecurr, 0) 
					(cos \anglenext, sin \anglenext, 0)};
				\addplot3[solid, pPntColor] coordinates { 
					(-cos \anglecurr, -sin \anglecurr, 0) 
					(-cos \anglenext, -sin \anglenext, 0)};
			}
			\foreach \j in {1, 2, ..., \nPoint}
			{
				\pgfmathsetmacro\anglecurr{\dualAngle[\j-1]}
				\pgfmathsetmacro\anglenext{\dualAngle[\j]}
				\pgfmathsetmacro\radiuscurr{\dualRadius[\j-1]}
				\pgfmathsetmacro\radiusnext{\dualRadius[\j]}
				
				\addplot3[solid, dPntColor] coordinates {
					(\radiuscurr * cos \anglecurr, \radiuscurr * sin \anglecurr, 0) 
					(\radiusnext * cos \anglenext, \radiusnext * sin \anglenext, 0)}; 
				\addplot3[solid, dPntColor] coordinates {
					(-1 * \radiuscurr * cos \anglecurr, -1 * \radiuscurr * sin \anglecurr, 0) 
					(-1 * \radiusnext * cos \anglenext, -1 * \radiusnext * sin \anglenext, 0)}; 
			}
			
			\def\gammaAngle{180}
			\addplot3[dashed, cColor] coordinates {
				(0,0,0) 
				(cos \gammaAngle, sin \gammaAngle, 0)};
			\addplot3[dashed, cColor] coordinates {
				(0,0,0)
				(cos \Gamma * cos \gammaAngle, cos \Gamma * sin \gammaAngle, sin \Gamma)};
			\addplot3[only marks,mark=*,mark size=0.6,cColor, fill=cColor] coordinates{(cos \gammaAngle, sin \gammaAngle, 0)};
			\addplot3[only marks,mark=*,mark size=0.6,cColor, fill=cColor] coordinates{(cos \Gamma * cos \gammaAngle, cos \Gamma * sin \gammaAngle, sin \Gamma)};
			\addplot3 [domain = 0:\Gamma, samples = 3, samples y=1, cColor]
			({0.3* cos \gammaAngle * cos(x)}, {0.3 * sin \gammaAngle * cos(x)}, {0.3* sin(x)});
			\node[above left, cColor] at (0.3* cos \gammaAngle * cos \Gamma, 0.3 * sin \gammaAngle * cos \Gamma, 0.3* sin \Gamma) {\tiny$\gamma_0$};
			\end{axis}
			\end{tikzpicture}
		}
		~~
		\subfloat[Illustration of the G-USC.\label{fig:geometry-3D-PRC}] 
		{
			\begin{tikzpicture}[scale=1.9]
			\begin{axis}[
			hide axis,
			view={\ViewAzimuth}{\ViewElevation},     
			every axis plot/.style={very thin},
			disabledatascaling,                      
			anchor=origin,                           
			viewport={\ViewAzimuth}{\ViewElevation}, 
			xmin=-3,   xmax=3,
			ymin=-3,   ymax=3,
			]
			
			\def\Gamma{30}
			\draw [black, fill=black] plot [mark=*, mark size=0.3] coordinates{(0, 0, 0)};
			\addplot3 [domain = 0:360, samples = 60, samples y=1]
			({cos(x)}, {sin(x)}, {0});
			\foreach \i in {0, 30,...,150}
			\addplot3 [domain = 0:360, samples = 60, samples y=1, draw opacity = 0.1, only foreground]
			({cos \i*sin(x)}, {sin \i*sin(x)}, {cos(x)});
			\foreach \i in {0, 30,...,150}
			\addplot3 [domain = 0:360, samples = 60, samples y=1, draw opacity = 0.1, only background, dashed]
			({cos \i*sin(x)}, {sin \i*sin(x)}, {cos(x)});
			\foreach \j in {30, 60, ..., 150}
			\addplot3 [domain = 0:360, samples = 60, samples y=1, draw opacity = 0.1, only foreground]
			({cos(x)*sin \j}, {sin(x)*sin \j}, {cos \j});
			\foreach \j in {30, 60, ..., 150}
			\addplot3 [domain = 0:360, samples = 60, samples y=1, draw opacity = 0.1, only background, dashed]
			({cos(x)*sin \j}, {sin(x)*sin \j}, {cos \j});
			\def\Sx{1.3}         
			\def\Sy{1.7}
			\filldraw[draw=none,fill=gray!20, opacity=0.2] 
			(-\Sx,-\Sy,0) -- (-\Sx, \Sy, 0) -- (\Sx, \Sy, 0) -- (\Sx,  -\Sy, 0) -- cycle;
			\node[black] at (1.1, 1.1, 0) {\tiny$\mathcal{S}_0$};
			\addplot3[domain=0:360, samples=51, samples y=11,smooth,
			domain y=-\Gamma:\Gamma,surf,shader=flat,color=pRgnColor,draw opacity = 0.0, fill opacity=0.5] 
			({1*cos(x)*cos(y)},
			{1*sin(x)*cos(y)}, {1*sin(y)});
			\foreach \j in {1, 2, ..., \nPoint}
			{
				\pgfmathsetmacro\angle{\primalAngle[\j-1]}
				\addplot3[only marks,mark=*,mark size=0.6,pPntColor, fill=pPntColor] coordinates { (cos \angle, sin \angle, 0) };
				\addplot3[only marks,mark=*,mark size=0.6,pPntColor, fill=pPntColor] coordinates { (-cos \angle, -sin \angle, 0) };
			}     
			\foreach \j in {1, 2, ..., \nPoint}
			{
				\pgfmathsetmacro\anglecurr{\primalAngle[\j-1]}
				\pgfmathsetmacro\anglenext{\primalAngle[\j]}
				\addplot3[solid, pPntColor] coordinates { 
					(cos \anglecurr, sin \anglecurr, 0) 
					(cos \anglenext, sin \anglenext, 0)};
				\addplot3[solid, pPntColor] coordinates { 
					(-cos \anglecurr, -sin \anglecurr, 0) 
					(-cos \anglenext, -sin \anglenext, 0)};
			}
			
			
			\def\gammaAngle{180}
			\addplot3[dashed, cColor] coordinates {
				(0,0,0) 
				(cos \gammaAngle, sin \gammaAngle, 0)};
			\addplot3[dashed, cColor] coordinates {
				(0,0,0)
				(cos \Gamma * cos \gammaAngle, cos \Gamma * sin \gammaAngle, sin \Gamma)};
			\addplot3[only marks,mark=*,mark size=0.6,cColor, fill=cColor] coordinates{(cos \gammaAngle, sin \gammaAngle, 0)};
			\addplot3[only marks,mark=*,mark size=0.6,cColor, fill=cColor] coordinates{(cos \Gamma * cos \gammaAngle, cos \Gamma * sin \gammaAngle, sin \Gamma)};
			\addplot3 [domain = 0:\Gamma, samples = 3, samples y=1, cColor]
			({0.3* cos \gammaAngle * cos(x)}, {0.3 * sin \gammaAngle * cos(x)}, {0.3* sin(x)});
			\node[above left, cColor] at (0.3* cos \gammaAngle * cos \Gamma, 0.3 * sin \gammaAngle * cos \Gamma, 0.3* sin \Gamma) {\tiny$\gamma_0$};
			\end{axis}
			\end{tikzpicture}
		}
		\caption[Illustration of the geometry associated with subspace-preserving representations]{\label{fig:geometry}Geometric interpretations of subspace-preserving conditions in \cref{thm:main}. The atoms $\cA_0:=\{\a_j\}_{j=1}^5$ (drawn in {\color{pPntColor}blue}) and the signal $\b$ (marked as $\times$, for (a) and (b) only) lie on the unit circle of a two-dimensional subspace $\cS_0$. The atoms $\cA_-$ (not drawn) lie on the unit sphere in the ambient space $\Re^3$. 
			In (a) and (b), points in $\cD(\A_0, \b)$ and $\cR(\A_0, \b)$ are illustrated as {\color{iPntColor}green} dots (see also \cref{fig:geometry-dual,fig:geometry-residual}). 
			In (c), points in $\cD(\A_0)$ are illustrated as {\color{dPntColor}red} dots (see also \cref{fig:geometry-2D}). 
			The G-IDC (resp., G-IRC, G-UDC and G-USC) holds if no point from $\cA_-$ lies in the {yellow} region in (a) (resp, (b), (c) and (d)), which corresponds to the set $\{\a \in \Re^3: \gamma_0 < \theta(\{\pm \v\}, \a)\}$ with $\v$ in the set $\cD(\A_0, \b)$ (resp., $\cR(\A_0, \b)$, $\cD(\A_0)$ and $\cS_0$). 
		}
	\end{figure*}
	
	\fi
	
	\subsection{Major results}

	We highlight our instance and universal recovery conditions that are characterized by the \emph{covering radius} and the \emph{angular distance} (i.e., the bottom rows of \cref{fig:instance-recovery-conditions,fig:result-flowchart}), and demonstrate their geometric interpretations.  
	The \emph{covering radius} $\gamma_0$, formally defined in \cref{def:covering-radius}, is the smallest angle such that closed spherical balls of that radius centered at the points of $\pm \cA_0$ cover the unit sphere of $\cS_0$. 
	This covering radius measures how well distributed the points from $\cA_0$ are in the subspace $\cS_0$, and should be relatively small if the points are equally distributed in all directions within the subspace and not skewed in a certain direction. 
	The \emph{angular distance} between two sets $\cV\subseteq\Re^D\setminus\{0\}$ and $\cW\subseteq\Re^D\setminus\{0\}$ is given by the minimum angular distance between points from these two sets, i.e.,
	$$
	\theta(\cV, \cW):= 
	\inf_{v\in\cV,w\in\cW} \cos^{-1} \left\langle \frac{\v}{\|\v\|_2}, \frac{\w}{\|\w\|_2} \right\rangle. 
	$$
	
	Our major results for subspace-preserving recovery are given in the following theorem, which encapsulates \cref{thm:G-IDC,thm:G-IRC,thm:G-UDC,thm:G-USC,thm:G-UDC-OMP}.
	\begin{theorem}\label{thm:main}
		The following geometric recovery conditions hold:
		\begin{enumerate}[label=(\roman*)]
			\item (Instance recovery conditions) For a fixed $\0 \ne \b \in \cS_0$, BP and OMP recover subspace-preserving solutions if G-IDC and G-IRC hold, respectively, where
			\begin{align}
				&\exists \v \in \cD(\A_0, \b)
				\ \ \text{satisfying} \ \
				\gamma_0 < \theta(\{\pm\v\}, \cA_-), ~\text{and}\tag{G-IDC}\\
				&\forall \v \in \cR(\A_0, \b)
				\ \text{it holds that} \
				\gamma_0 < \theta(\{\pm\v\}, \cA_-),\tag{G-IRC}
			\end{align}
			where $\cD(\A_0, \b)\subseteq \cS_0$ and $\cR(\A_0, \b)\subseteq \cS_0$ are the sets of dual points (see \cref{def:instance-dual-point}) and residual points (see \cref{def:residual-points}), respectively. 
			\item (Universal recovery conditions) BP and OMP recover subspace-preserving solutions
			for all $\0 \ne \b \in \cS_0$ 
		 if either G-UDC holds or G-USC holds, where
			\begin{align}
				&\forall \v \in \cD(\A_0) 
				\ \text{it holds that} \
				\gamma_0 < \theta(\{\pm\v\}, \cA_-), ~\text{and}\tag{G-UDC}\\
				&\forall \v \in \cS_0 
				\ \text{it holds that} \
				\gamma_0 < \theta(\{\pm\v\}, \cA_-),\tag{G-USC}
			\end{align}
			where $\cD(\A_0)\subseteq \cS_0$ is the set of dual points associated to $\A_0$ (see \cref{def:dual-point}).
		\end{enumerate}
	\end{theorem}
	
	Conditions G-IDC, G-IRC, G-UDC and G-USC in \cref{thm:main} all require the covering radius to be smaller than the angular distance between $\cA_-$ and a pair of points $\{\pm \v\}$, i.e., they require
	\begin{equation}\label{eq:covering-radius-angular-distance}
	\gamma_0 < \theta(\{\pm\v\}, \cA_-),
	\end{equation}
	for $\v$ that is chosen from a particular subset of $\cS_0$ (that varies for the four conditions).  
	Such conditions conform with our intuition that for subspace-preserving recovery to succeed, points in $\cA_0$ should be well distributed in the subspace $\cS_0$ and points in $\cA_-$ should be well separated from a subset of $\cS_0$.
	The only difference among the four conditions is that \cref{eq:covering-radius-angular-distance} needs to be satisfied for different $\v$'s. 
	
	In \cref{fig:geometry} we illustrate the G-IDC, G-IRC, G-UDC and G-USC using a two dimensional subspace $\cS_0$ in $\Re^3$. 
	The atoms in $\cA_0$ are illustrated as blue dots on the unit sphere of $\cS_0$. 
	The atoms in $\cA_-$ (not drawn) are points that lie on the unit sphere of $\Re^3$ but outside of $\cS_0$.
	The instance recovery conditions G-IDC and G-IRC depend on a specific vector $\b \in \cS_0$, which is illustrated by the symbol $\times$ in \cref{fig:geometry-G-IDC,fig:geometry-G-IRC}.
	\begin{enumerate}[label=(\roman*)]
		\item In \cref{fig:geometry-G-IDC}, the G-IDC holds if and only if there exists a point $\v \in \cD(\A_0, \b)$ such that no points from $\cA_-$ lie in the yellow region composed of a pair of antipodal spherical caps centered at $\frac{\pm\v}{\|\v\|_2}$ with radius $\gamma_0$.
		In this instance, $\cD(\A_0, \b)$ contains a single point illustrated as an orange dot (see also \cref{fig:geometry-dual}).
		\item In \cref{fig:geometry-G-IRC}, the G-IRC holds if and only if no points from $\cA_-$ lie in the yellow region composed of a union of antipodal spherical caps centered at $\frac{\pm\v}{\|\v\|_2}$ for $\v \in\cR(\A_0, \b)$ with radius $\gamma_0$.
		In this instance, $\cR(\A_0, \b)$ contains $d_0$ points illustrated as orange dots (see also \cref{fig:geometry-residual}).
		\item In \cref{fig:geometry-3D-DRC}, the G-UDC holds if and only if no points from $\cA_-$ lie in the yellow region composed of a union of spherical caps centered at $\frac{\pm\v}{\|\v\|_2}$ for $\v \in\cD(\A_0)$ with radius $\gamma_0$.
		The points in $\cD(\A_0)$ are illustrated as red dots (see also \cref{fig:geometry-2D}).
		\item In \cref{fig:geometry-3D-PRC}, the G-USC holds if and only if no points from $\cA_-$ lie in the yellow region enclosed by the equal latitude lines $\pm \gamma_0$ on the unit sphere. 
	\end{enumerate}
	Therefore, roughly speaking, subspace-preserving recovery requires the points $\cA_-$ to be separated from a specific region of the subspace $\cS_0$ that is determined by both $\cA_0$ and $\b$ for instance recovery, and by only $\cA_0$ for universal recovery.
	
	\subsection{Paper outline}	
	The remainder of the paper is organized as follows.
	In \cref{sec:background}
	we introduce inradius, covering radius and circumradius for characterizing the distribution of points in $\cS_0$, as well as coherence for measuring the separation between points from $\cA_-$ and $\cS_0$. 
	These notions are used to formulate the instance and universal recovery conditions, which are derived in \cref{sec:instance_recovery} and \cref{sec:universal_recovery}, respectively. 
	In \cref{sec:sparse-recovery}, we discuss the relationship between our subspace-preserving recovery conditions with a dual certificate, mutual incoherence, exact recovery and null space condition that are well known in the sparse signal recovery literature. 
	Conclusions and future research are given in \cref{sec:conclusions}.

	\section{Background Material}
	\label{sec:background}
	In this section we present the background material on the geometric characterization of the dictionary $\cA$  required for our analysis in \cref{sec:geometric_characterization}, and an overview of the BP and OMP methods in \cref{sec:BPOMP}.  
	
	\subsection{Geometric characterization of the dictionary $\cA$}
	\label{sec:geometric_characterization}
	
	Our geometric conditions for subspace-preserving recovery rely on geometric properties that characterize the distribution of points $\cA_0$ in the subspace $\cS_0$ and the separation between $\cA_-$ and the subspace $\cS_0$. 
	In this section, we introduce the relevant concepts and notations.
	
	The convex hull of the symmetrized points in $\cA_0$ is defined as 
	\begin{equation*} 
		\cK_0 := 
		\conv(\pm \cA_0),
	\end{equation*}
	where $\conv(\cdot)$ denotes the convex hull, and its (relative) polar set is defined as
	\begin{equation*}
		\begin{split}
			\cK_0^o &:= 
			\{\v \in \cS_0: |\langle \v, \a \rangle|\le 1 \ \text{for all} \ \a \in \cK_0\}
			\equiv \{\v \in \cS_0: |\langle \v, \a \rangle|\le 1 \ \text{for all} \ \a \in \pm\cA_0\},
		\end{split}
	\end{equation*}
	where the second equality follows from \cref{thm:polar-set-equivalent}. Both $\cK_0$ and $\cK^o_0$ are symmetric convex bodies\footnote{If a convex set $\cP$  satisfies $\cP = -\cP$, then we say it is symmetric. A compact convex set with a nonempty interior is called a convex body.}, as the polar of a convex body is also a convex body \cite{Brazitikos:14}.

	Next, we introduce three concepts for characterizing the distribution of the data points in $\cA_0$. 
	The first concept is the inradius of a convex body.
	\begin{definition}[inradius]
		\label{def:inradius}
		The (relative) inradius of a convex body $\cP$, denoted by $r(\cP)$, is the radius of the largest Euclidean ball in $\spann(\cP)$ inscribed in $\cP$.
	\end{definition}
	
	In our analysis,  we use the inradius 
	$$
	r_0 := r(\cK_0),
	$$
	which characterizes the distribution of the data points $\cA_0$ in $\cS_0$.  If the elements of $\cA_0$ are well distributed in $\cS_0$, the inradius is near one; otherwise, the inradius is small.
	
	The second concept for quantifying the distribution of points is given by the covering radius.  
	Let $\Sp ^{D-1} := \{\v\in \Re^D: \|\v\|_2 = 1\}$ be the unit sphere of $\Re^D$.
	The (relative) covering radius of a set is then defined as follows. 
	\begin{definition}[covering radius]\label{def:covering-radius}
		The (relative) covering radius of $\cV\subseteq\Re^D$ is 
		\[
		\gamma(\cV):= \sup \{ \theta(\cV, \w): \w\in \spann(\cV) \cap \Sp^{D-1} \}.
		\]
	\end{definition}
	Our analysis uses the covering radius of the set $\pm\cA_0$, namely
	$$
	\gamma_0 := \gamma(\pm \cA_0),
	$$
	which by \cref{def:covering-radius} is 
	computed by finding $\w \in \Sp^{D-1}$ that is furthest away from all points in $\pm \cA_0$ (see \cref{fig:geometry-2D}). 
	It can also be interpreted as the smallest radius such that closed spherical balls of that radius centered at the points of $\pm \cA_0$ cover $\Sp ^{D-1} \cap \cS_0$.
	The covering radius characterizes how well the points in $\pm \cA_0$ are distributed.
	
	\ifdefined\draft
	\else
	
	\begin{figure*}[h]
		\centering
		\def\nPoint{5}
		\def\primalAngle
		{{30.00, 57.60, 86.40, 126.00, 150.00, 210.00}}
		\def\dualRadius
		{{1.030, 1.032, 1.063, 1.022, 1.154, 1.030}}
		\def\dualAngle
		{{43.80, 72.00, 106.20, 138.00, 180.00, 223.80}}
		\subfloat[Inradius and covering radius]
		{
			\begin{tikzpicture}[scale = 2.2]
			\coordinate (0) at (0,0);
			\def\wAngle{0};
			\def\wRefAngle{\primalAngle[0]};
			
			\def\Sx{1.2}         
			\def\Sy{1.2}
			\filldraw[draw=none,fill=gray!20, opacity=0.2] 
			(-\Sx,-\Sy) -- (-\Sx, \Sy) -- (\Sx, \Sy) -- (\Sx, -\Sy) -- cycle;
			\node[black] at (1.1, 1.1, 0) {$\cS_0$};
			\draw[black, fill = none] (0) circle [radius = 1]; 
			\node[left, black] at (0) {O}; 
			\draw [black, fill=black] (0) circle [radius=0.02];
			\foreach \i in {1, 2, ..., \nPoint}
			{
				\def\angle{\primalAngle[\i-1]}
				\draw [pPntColor, fill=pPntColor] (cos \angle, sin \angle) circle [radius=0.02];
				\draw [pPntColor, fill=pPntColor] (-cos \angle, -sin \angle) circle [radius=0.02];
				\draw [dotted, pPntColor] (cos \angle, sin \angle) -- (-cos \angle, -sin \angle);
				\node[above, pPntColor] at (cos \angle, sin \angle) {$\a_\i$}; 
				\node[below, pPntColor] at (-cos \angle, -sin \angle) {$-\a_\i$}; 
			}
			\foreach \i in {1, 2, ..., \nPoint}
			{
				\draw[solid, pPntColor] (cos \primalAngle[\i-1], sin \primalAngle[\i-1]) -- (cos \primalAngle[\i], sin \primalAngle[\i]); 
				\draw[solid, pPntColor] (-cos \primalAngle[\i-1], -sin \primalAngle[\i-1]) -- (-cos \primalAngle[\i], -sin \primalAngle[\i]); 
			}
			\node[right, pPntColor] at (-0.9, -0.1) {$\cK_0$};
			\node[above right, cColor] at (cos \wAngle, sin \wAngle) {$\w^*$};
			\draw [cColor, fill=cColor] (cos \wAngle, sin \wAngle) circle [radius=0.02];
			\draw[dotted, cColor] (0) -- (cos \wAngle, sin \wAngle);
			\draw[thick, solid, cColor] (0) ([shift=(\wAngle:0.2cm)] 0, 0) arc (\wAngle:\wRefAngle:0.2cm);
			\node[right, cColor] at (0.18, 0.07) {$\gamma_0$};
			\draw[dashed, cColor, fill = none] (0) circle [radius = cos \wRefAngle]; 
			\draw [cColor, fill=cColor] (cos \wRefAngle * cos 66, cos \wRefAngle * sin 66) circle [radius=0.02];
			\draw[solid, cColor] (0) -- (cos \wRefAngle * cos 66, cos \wRefAngle * sin 66);
			\draw[thick,cColor,decoration={brace,raise=2pt,amplitude=10pt},decorate]
			(0,0) -- node[left] {$r_0\;\;\;$} (66:cos \wRefAngle);
			\end{tikzpicture}
		}
		~~~
		\subfloat[Circumradius]
		{
			\begin{tikzpicture}[scale = 2.2]
			\coordinate (0) at (0,0);
			\def\wAngle{0};
			\def\wRefAngle{\primalAngle[0]};
			
			\def\Sx{1.2}         
			\def\Sy{1.2}
			\filldraw[draw=none,fill=gray!20, opacity=0.2] 
			(-\Sx,-\Sy) -- (-\Sx, \Sy) -- (\Sx, \Sy) -- (\Sx, -\Sy) -- cycle;
			\node[black] at (1.1, 1.1, 0) {$\cS_0$};
			\draw[black, fill = none] (0) circle [radius = 1]; 
			\node[left, black] at (0) {O}; 
			\draw [black, fill=black] (0) circle [radius=0.02];
			\foreach \i in {1, 2, ..., \nPoint}
			{
				\def\angle{\primalAngle[\i-1]}
				\draw [pPntColor, fill=pPntColor] (cos \angle, sin \angle) circle [radius=0.02];
				\draw [pPntColor, fill=pPntColor] (-cos \angle, -sin \angle) circle [radius=0.02];
				\draw [dotted, pPntColor] (cos \angle, sin \angle) -- (-cos \angle, -sin \angle);
				\node[above, pPntColor] at (cos \angle, sin \angle) {$\a_\i$}; 
				\node[below, pPntColor] at (-cos \angle, -sin \angle) {$-\a_\i$}; 
			}
			\foreach \i in {1, 2, ..., \nPoint}
			{
				\def\angle{\dualAngle[\i-1]}
				\def\radius{\dualRadius[\i-1]}
				\draw [dPntColor, fill=dPntColor] (\radius * cos \angle, \radius * sin \angle) circle [radius=0.03];
				\draw [dPntColor, fill=dPntColor] (-1 * \radius * cos \angle, -1 * \radius * sin \angle) circle [radius=0.03];
			}
			\foreach \i in {1, 2, ..., \nPoint}
			{
				\draw[solid, pPntColor] (cos \primalAngle[\i-1], sin \primalAngle[\i-1]) -- (cos \primalAngle[\i], sin \primalAngle[\i]); 
				\draw[solid, pPntColor] (-cos \primalAngle[\i-1], -sin \primalAngle[\i-1]) -- (-cos \primalAngle[\i], -sin \primalAngle[\i]); 
			}
			\node[right, pPntColor] at (-0.9, 0.1) {$\cK_0$};
			\foreach \i in {1, 2, ..., \nPoint}
			{
				\draw[solid, dPntColor] (\dualRadius[\i-1] * cos \dualAngle[\i-1], \dualRadius[\i-1] * sin \dualAngle[\i-1]) -- (\dualRadius[\i] * cos \dualAngle[\i], \dualRadius[\i] * sin \dualAngle[\i]); 
				\draw[solid, dPntColor] (-1 * \dualRadius[\i-1] * cos \dualAngle[\i-1], -1 * \dualRadius[\i-1] * sin \dualAngle[\i-1]) -- (-1 * \dualRadius[\i] * cos \dualAngle[\i], -1 * \dualRadius[\i] * sin \dualAngle[\i]); 
			}
			\node[left, dPntColor] at (-0.95, -0.3) {$\cK_0^o$};
			\draw[dashed, cColor, fill = none] (0) circle [radius = 1.0 / cos \wRefAngle];
			\draw [cColor, fill=cColor] (1.0 / cos \wRefAngle * cos 66, 1.0 / cos \wRefAngle * sin 66) circle [radius=0.02];
			\draw[solid, cColor] (0) -- (1.0 / cos \wRefAngle * cos 66, 1.0 / cos \wRefAngle * sin 66);
			\draw[thick,cColor,decoration={brace,raise=2pt,amplitude=10pt},decorate]
			(0,0) -- node[left] {$R_0\;\;\;$} (66:1/cos \wRefAngle);
			\end{tikzpicture}
		}
		\caption[Illustration of the geometric characterizations of the dictionary $\cA$]{\label{fig:geometry-2D}Illustration of the geometric characterizations of the dictionary $\cA$. In this example, the data set $\cA_0:=\{\a_j\}_{j=1}^5$ lies on the unit circle of a two-dimensional subspace $\cS_0$. 
			The sets $\cK_0$ and $\cK^o_0$ are illustrated as the {\color{pPntColor}blue} and {\color{dPntColor}red} polygons, respectively.
			(a) The inradius $r_0$ is the radius of the inscribing ball of $\cK_0$ shown as the {\color{cColor}orange} dashed circle. 
			The covering radius $\gamma_0$ is the angle $\theta(\pm\cA_0, \w^*)$ where $\w^*$ is the maximizer of the optimization problem in \cref{def:covering-radius}.	%
			(b) The circumradius $R_0$ is the radius of the smallest ball containing $\cK^o_0$ shown as the {\color{cColor}orange} dashed circle.
		}
	\end{figure*}
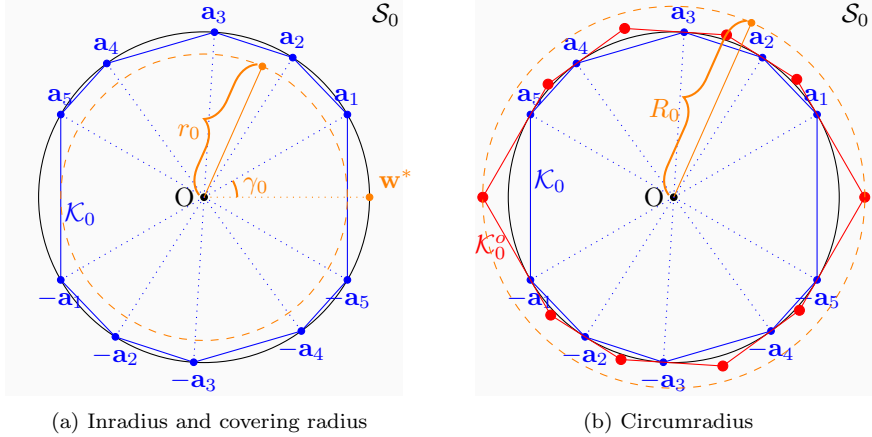
	
	\fi
	
	The third characterization of the distribution of $\cA_0$ is in terms of the circumradius.
	
	\begin{definition}[circumradius]
		\label{def:circumradius}
		The circumradius $R(\cP)$ of a convex body $\cP$ is defined as the radius of the smallest Euclidean ball containing $\cP$.
	\end{definition}
	Our analysis uses the circumradius of the polar set $\cK^o_0$, 
	namely
	$$
	R_0 := R(\cK^o_0).
	$$
	The next result shows the relationship between  
	the circumradius $R_0$, 
	the inradius $r_0$, 
	and the covering radius $\gamma_0$  
	(see \cref{fig:geometry-2D}), and is proved in the appendix.
	\begin{theorem}
		\label{thm:inradius-coveringradius-circumradius}
		It holds that
		$r_0 = \cos(\gamma_0) = 1/R_0$.
	\end{theorem}
	
	Finally, we measure the separation between $\cA_-$ and an arbitrary subset of $\cS_0$ using their coherence, which is
	the maximum inner product (in absolute value) between points from the two sets. That is, the coherence between $\cA_-$ and $\cW \subseteq \cS_0\setminus{\{0\}}$ is  
	\begin{equation}\label{eq:def-coherence}
	\mu(\cW, \cA_-) := \max_{\w \in \cW}\max_{\a \in \cA_-} \left|\left\langle\frac{\w}{\|\w\|_2}, \a\right\rangle\right|.
	\end{equation}

	\subsection{Overview of BP and OMP}
	\label{sec:BPOMP}
	
	BP and OMP are two of the most popular methods for finding sparse solutions to an under-determined system of linear equations.
	Given the dictionary matrix $\A$ and a signal $\b$, BP and OMP aim to solve
	\begin{equation}
	\min_{\c\in\Re^N} \|\c\|_0 \st \A \c = \b.
	\label{eq:sparse}
	\end{equation}
	It is well-known that finding the solution to \cref{eq:sparse} is an NP-hard problem in general.
	
	The BP method is a convex relaxation approach that replaces the $\ell_0$-regularization in \cref{eq:sparse} by an $\ell_1$-norm, whose solution set we denote by $\BP(\A, \b)$, i.e., 
	\begin{equation}
	\BP(\A, \b) 
	:= \Argmin_{\c\in\Re^{N}} \|\c\|_1 \st \A \c = \b.
	\label{eq:def-BP}
	\end{equation}
	The convex optimization problem~\cref{eq:def-BP} can be efficiently solved using existing solvers. 
	Under the problem setup in \cref{sec:problem_formulation}, there always exists a subspace-preserving representation $\c$ satisfying $\A \c = \b$ so that the set $\BP(\A, \b)$ is non-empty. We use a capitalized ``Argmin" in \cref{eq:def-BP} to indicate that $\BP(\A,\b)$ is a set of solutions. 

	
	The OMP method is a greedy strategy that sequentially chooses dictionary elements  in a locally optimal manner. Specifically, OMP maintains a working set (i.e., a subset of the dictionary) that for the $k$th iteration is denoted by $\cW^{(k)}$ in \cref{alg.omp}. Using the $k$th working set, a representation vector $\c^{(k)}$ with support corresponding to $\cW^{(k)}$ is computed in~\cref{eq:OMP-residual}. 
	The next subset $\cW^{(k+1)}$ is computed by adding to $\cW^{(k)}$ a dictionary element with maximum absolute inner product with the $k$th constraint residual $\v^{(k)}$ computed in Step~\ref{step.vk}. For convenience, let us define
	$$
	\OMP(\A, \b) := \text{the vector $\c^{(k)}$ returned in Step~\ref{step.return} of the $\OMP$ \cref{alg.omp}.} 
	$$
	Note that the constraint residual $\v^{(k)}$ computed in Step~\ref{step.vk} is the projection of $\b$ onto the orthogonal complement of the space spanned by the columns of $\A$ indexed by $\cW^{(k)}$.  It follows that the vector added to $\cW^{(k)}$ in Step~\ref{eq:OMP-maximizer} is linearly independent from the columns of $\A$ indexed by $\cW^{(k)}$. Thus, the solution $\c^{(k)}$ computed in \cref{eq:OMP-residual} is the unique solution to the optimization problem. 
	Finally, we comment that some fixed  strategy should be used when the optimization problem in \cref{eq:OMP-maximizer} does not have a unique minimizer, such as adding the minimizer $\a_j$ with smallest index $j$.   
	
	In a practical implementation of \cref{alg.omp},
	termination would occur when either $\|\v^{(k)}\|_2 \leq \epsilon$ for some  tolerance $\epsilon \in[0,\infty)$ or when the iteration $k$ reaches a maximum allowed value, say $k_{\max}$. For the analysis and for simplicity of presentation in \cref{alg.omp}, we assume $\epsilon = 0$ and $k_{\max} = \infty$, and note that termination will always occur (assuming exact arithmetic) for the problem setup described in \cref{sec:problem_formulation}.
	
	\begin{algorithm}[ht]
		\caption{Orthogonal Matching Pursuit}
		\label{alg.omp}
		\begin{algorithmic}[1]
			\State \textbf{input:} $\A$, $\b$.
			\State 
			Set $\cW^{(0)} \leftarrow \emptyset$.
			\For{$k=0,1,2,\dots$}
			\State Compute the representation vector corresponding to the working set $\cW^{(k)}$ as
			\begin{equation}\label{eq:OMP-residual}
			\c^{(k)} := \argmin_{\c\in\Re^N} \|\b - \A \c\|_2 \st \c_j = 0 ~\text{if}~ \a_j \notin \cW^{(k)}.
			\end{equation}   
			\State Compute the residual vector $\v^{(k)} := \b-\A\c^{(k)}$. \label{step.vk}
			\State \textbf{if $\|\v^{(k)}\| = 0$} \textbf{then} return $\OMP(\A,\b) := \c^{(k)}$ \textbf{end if} \label{step.return}
			\State Define the working set for the next iteration as  
			\begin{equation}\label{eq:OMP-maximizer}
			\cW^{(k+1)} := \cW^{(k)} \cup \argmax_{\a_j \in \cA} |\langle \a_j, \v^{(k)} \rangle|.
			\end{equation}	
			\EndFor 
		\end{algorithmic}
	\end{algorithm}

	\section{Instance Recovery Conditions}
	\label{sec:instance_recovery}
	
	We present conditions under which BP and OMP produce subspace-preserving solutions 
	for a \emph{fixed} $\0 \ne \b \in \cS_0$. 
	A summary of the instance recovery conditions is given in~\cref{fig:instance-recovery-conditions}.
	As shown in \cref{sec:instance_recovery_BP} and \cref{sec:instance_recovery_OMP}, our conditions are characterized by the distribution of points $\cA_0$ in the subspace $\cS_0$ and the separation of $\cA_0$ from a set of \emph{dual points} (for BP) and a set of \emph{residual points} (for OMP).
	In \cref{sec:comparison_BP_OMP} we discuss how these conditions are related.
	
	\subsection{Instance recovery conditions for BP}
	\label{sec:instance_recovery_BP}
	
	\ifdefined\draft
	\else
	
	\begin{figure*}[t]
		\centering
		\def\nPoint{5}
		\def\primalAngle
		{{30.00, 57.60, 86.40, 126.00, 150.00, 210.00}}
		\def\dualRadius
		{{1.030, 1.032, 1.063, 1.022, 1.154, 1.030}}
		\def\dualAngle
		{{43.80, 72.00, 106.20, 138.00, 180.00, 223.80}}
		\def\instancePoint{195}
		\def\instanceDualRadius{1.154}
		\def\instanceDualAngle{180.00}
		\subfloat[Illustration in 2D.\label{fig:geometry-2D-Dual}]
		{
			\begin{tikzpicture}[scale = 1.8]
			\coordinate (0) at (0,0);
			\def\wAngle{0};
			\def\wRefAngle{\primalAngle[0]};
			
			\def\Sx{1.2}         
			\def\Sy{1.2}
			\filldraw[draw=none,fill=gray!20, opacity=0.2] 
			(-\Sx,-\Sy) -- (-\Sx, \Sy) -- (\Sx, \Sy) -- (\Sx, -\Sy) -- cycle;
			\node[black] at (1.1, 1.1, 0) {$\cS_0$};
			\draw[black, fill = none] (0) circle [radius = 1]; 
			\node[left, black] at (0) {O}; 
			\draw [black, fill=black] (0) circle [radius=0.02];
			\foreach \i in {1, 2, ..., \nPoint}
			{
				\def\angle{\primalAngle[\i-1]}
				\draw [pPntColor, fill=pPntColor] (cos \angle, sin \angle) circle [radius=0.02];
				\draw [pPntColor, fill=pPntColor] (-cos \angle, -sin \angle) circle [radius=0.02];
				\draw [dotted, pPntColor] (cos \angle, sin \angle) -- (-cos \angle, -sin \angle);
				\node[above, pPntColor] at (cos \angle, sin \angle) {$\a_\i$}; 
				\node[below, pPntColor] at (-cos \angle, -sin \angle) {$-\a_\i$}; 
			}
			\foreach \i in {1, 2, ..., \nPoint}
			{
				\def\angle{\dualAngle[\i-1]}
				\def\radius{\dualRadius[\i-1]}
				\draw [dPntColor, fill=dPntColor] (\radius * cos \angle, \radius * sin \angle) circle [radius=0.03];
				\draw [dPntColor, fill=dPntColor] (-1 * \radius * cos \angle, -1 * \radius * sin \angle) circle [radius=0.03];
			}
			\foreach \i in {1, 2, ..., \nPoint}
			{
				\draw[solid, pPntColor] (cos \primalAngle[\i-1], sin \primalAngle[\i-1]) -- (cos \primalAngle[\i], sin \primalAngle[\i]); 
				\draw[solid, pPntColor] (-cos \primalAngle[\i-1], -sin \primalAngle[\i-1]) -- (-cos \primalAngle[\i], -sin \primalAngle[\i]); 
			}
			\foreach \i in {1, 2, ..., \nPoint}
			{
				\draw[solid, dPntColor] (\dualRadius[\i-1] * cos \dualAngle[\i-1], \dualRadius[\i-1] * sin \dualAngle[\i-1]) -- (\dualRadius[\i] * cos \dualAngle[\i], \dualRadius[\i] * sin \dualAngle[\i]); 
				\draw[solid, dPntColor] (-1 * \dualRadius[\i-1] * cos \dualAngle[\i-1], -1 * \dualRadius[\i-1] * sin \dualAngle[\i-1]) -- (-1 * \dualRadius[\i] * cos \dualAngle[\i], -1 * \dualRadius[\i] * sin \dualAngle[\i]); 
			}
			\draw (cos \instancePoint, sin \instancePoint) node[cross] {};
			\node[right, black] at (cos \instancePoint, sin \instancePoint) {$\b$}; 
			\draw [iPntColor, fill=iPntColor] (\instanceDualRadius * cos \instanceDualAngle, \instanceDualRadius * sin \instanceDualAngle) circle [radius=0.04];
			\node[above, iPntColor] at (\instanceDualRadius * cos \instanceDualAngle, \instanceDualRadius * sin \instanceDualAngle) {$\cD(\A_0,\b)$}; 
			\end{tikzpicture}
		}
		~~
		\subfloat[Illustration in 3D.\label{fig:geometry-3D-Dual}]
		{
			\begin{tikzpicture}[scale=2.0]
			\begin{axis}[
			hide axis,
			view={\ViewAzimuth}{\ViewElevation},     
			every axis plot/.style={very thin},
			disabledatascaling,                      
			anchor=origin,                           
			viewport={\ViewAzimuth}{\ViewElevation}, 
			xmin=-3,   xmax=3,
			ymin=-3,   ymax=3,
			]
			
			\def\Gamma{30}
			\draw [black, fill=black] plot [mark=*, mark size=0.3] coordinates{(0, 0, 0)};
			\addplot3 [domain = 0:360, samples = 60, samples y=1]
			({cos(x)}, {sin(x)}, {0});
			\foreach \i in {0, 30,...,150}
			\addplot3 [domain = 0:360, samples = 60, samples y=1, draw opacity = 0.1, only foreground]
			({cos \i*sin(x)}, {sin \i*sin(x)}, {cos(x)});
			\foreach \i in {0, 30,...,150}
			\addplot3 [domain = 0:360, samples = 60, samples y=1, draw opacity = 0.1, only background, dashed]
			({cos \i*sin(x)}, {sin \i*sin(x)}, {cos(x)});
			\foreach \j in {30, 60, ..., 150}
			\addplot3 [domain = 0:360, samples = 60, samples y=1, draw opacity = 0.1, only foreground]
			({cos(x)*sin \j}, {sin(x)*sin \j}, {cos \j});
			\foreach \j in {30, 60, ..., 150}
			\addplot3 [domain = 0:360, samples = 60, samples y=1, draw opacity = 0.1, only background, dashed]
			({cos(x)*sin \j}, {sin(x)*sin \j}, {cos \j});
			\def\Sx{1.3}         
			\def\Sy{1.7}
			\filldraw[draw=none,fill=gray!20, opacity=0.2] 
			(-\Sx,-\Sy,0) -- (-\Sx, \Sy, 0) -- (\Sx, \Sy, 0) -- (\Sx,  -\Sy, 0) -- cycle;
			\node[black] at (1.1, 1.1, 0) {\tiny$\mathcal{S}_0$};
			\foreach \j in {1, 2, ..., \nPoint}
			{
				\pgfmathsetmacro\angle{\primalAngle[\j-1]}
				\addplot3[only marks,mark=*,mark size=0.6,pPntColor, fill=pPntColor] coordinates { (cos \angle, sin \angle, 0) };
				\addplot3[only marks,mark=*,mark size=0.6,pPntColor, fill=pPntColor] coordinates { (-cos \angle, -sin \angle, 0) };
			}     
			\foreach \j in {1, 2, ..., \nPoint}
			{
				\pgfmathsetmacro\angle{\dualAngle[\j-1]}
				\pgfmathsetmacro\radius{\dualRadius[\j-1]}
				\addplot3[only marks,mark=*,mark size=0.6,dPntColor, fill=dPntColor] coordinates { (\radius * cos \angle, \radius * sin \angle, 0) };
				\addplot3[only marks,mark=*,mark size=0.6,dPntColor, fill=dPntColor] coordinates { (-\radius * cos \angle, -\radius * sin \angle, 0) };
			}
			\foreach \j in {1, 2, ..., \nPoint}
			{
				\pgfmathsetmacro\anglecurr{\primalAngle[\j-1]}
				\pgfmathsetmacro\anglenext{\primalAngle[\j]}
				\addplot3[solid, pPntColor] coordinates { 
					(cos \anglecurr, sin \anglecurr, 0) 
					(cos \anglenext, sin \anglenext, 0)};
				\addplot3[solid, pPntColor] coordinates { 
					(-cos \anglecurr, -sin \anglecurr, 0) 
					(-cos \anglenext, -sin \anglenext, 0)};
			}
			\foreach \j in {1, 2, ..., \nPoint}
			{
				\pgfmathsetmacro\anglecurr{\dualAngle[\j-1]}
				\pgfmathsetmacro\anglenext{\dualAngle[\j]}
				\pgfmathsetmacro\radiuscurr{\dualRadius[\j-1]}
				\pgfmathsetmacro\radiusnext{\dualRadius[\j]}
				
				\addplot3[solid, dPntColor] coordinates {
					(\radiuscurr * cos \anglecurr, \radiuscurr * sin \anglecurr, 0) 
					(\radiusnext * cos \anglenext, \radiusnext * sin \anglenext, 0)}; 
				\addplot3[solid, dPntColor] coordinates {
					(-1 * \radiuscurr * cos \anglecurr, -1 * \radiuscurr * sin \anglecurr, 0) 
					(-1 * \radiusnext * cos \anglenext, -1 * \radiusnext * sin \anglenext, 0)}; 
			}
			\addplot3[only marks,mark=x,mark size=2, black, fill=black] coordinates { (cos \instancePoint, sin \instancePoint, 0.0) };
			\node[black, right] at (cos \instancePoint, sin \instancePoint, 0.0)  {\tiny$\b$}; 
			\addplot3[only marks,mark=*,mark size=1, iPntColor, fill=iPntColor] coordinates { (\instanceDualRadius * cos \instanceDualAngle, \instanceDualRadius * sin \instanceDualAngle, 0) };
			\addplot3[solid, thick, cColor] coordinates { 
				(\instanceDualRadius * cos \instanceDualAngle, \instanceDualRadius * sin \instanceDualAngle, -1) 
				(\instanceDualRadius * cos \instanceDualAngle, \instanceDualRadius * sin \instanceDualAngle, 1) };
			\node[iPntColor, right] at (\instanceDualRadius * cos \instanceDualAngle, \instanceDualRadius * sin \instanceDualAngle, 0.0)  {\tiny$\cD(\A_0,\b)$}; 
			\node[cColor, right] at (\instanceDualRadius * cos \instanceDualAngle, \instanceDualRadius * sin \instanceDualAngle, 1)  {\tiny$\BP_{\Dual}(\A_0,\b)$}; 
			\end{axis}
			\end{tikzpicture}
		}
		\caption[Illustration of the set of dual points and dual set]{\label{fig:geometry-dual}Illustration of the set of dual points $\cD(\A_0, \b)$ (see \cref{def:instance-dual-point}) and the set of dual solutions $\BP_{\Dual}(\A_0,\b)$. 
			The atoms $\cA_0:=\{\a_j\}_{j=1}^5$ (drawn in {\color{pPntColor}blue}) and $\b$ (marked as $\times$) lie on the unit circle of a two-dimensional subspace $\cS_0$. 
			$\cD(\A_0, \b)$ contains points from $\cK_0^o$ that maximize the inner product with $\b$. 
			$\BP_{\Dual}(\A_0,\b)$ contains points whose orthogonal projection onto $\cS_0$ lies in $\cD(\A_0, \b)$ (see \cref{lem:B-shifted-B}).
		}
	\end{figure*}
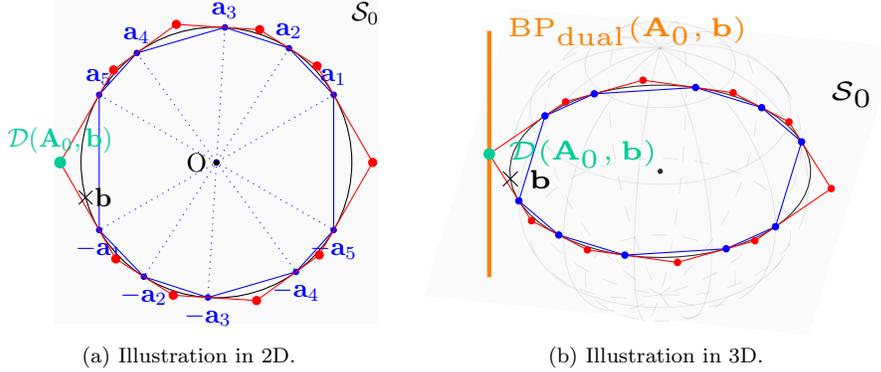
	
	\fi

	In this section we consider the instance recovery problem for BP.  In \cref{sec.tight.bp} we give a tight condition for recovery (i.e., a condition that is equivalent to the statement that BP always returns subspace-preserving solutions), in \cref{sec.suff.bp} we give a sufficient condition for ensuring recovery, and in \cref{sec.geom.bp} we give a weaker sufficient condition for ensuring recovery that has a clear geometric interpretation.  
	
	\subsubsection{A tight condition}\label{sec.tight.bp}
	
	The BP approach seeks a subspace-preserving representation by solving the optimization problem in \cref{eq:def-BP}. 
	From the optimality conditions, if $\c^* \in \BP(\A, \b)$ is optimal for \cref{eq:def-BP}, then there exists a $\v^*\in\Re^D$ so that 
	\begin{equation}\label{eq:ell1_optimality}
	\A^\transpose \v^* \in \partial \|\c^*\|_1,
	\end{equation}
	where $\partial \|\c^*\|_1 = \{\w: \|\w\|_\infty \le 1$ and $w_j = \text{sgn}(c_j^*)$ for $c_j \ne 0\}$ is the subdifferential of $\|\c\|_1$ at $\c^*$.
	In particular, $\v^*\in\BP_{\Dual}(\A,\b)$, where $\BP_{\Dual}(\A,\b)$ is the set of solutions to the dual optimization problem for \cref{eq:def-BP} given by
	\begin{equation}
	\BP_{\Dual}(\A,\b)
	:= \Argmax_{\v\in\Re^D} \langle \v, \b \rangle \st \ \|\A ^\transpose \v\|_{\infty} \le 1.
	\label{eq:def-D}
	\end{equation}
	From the optimality condition \cref{eq:ell1_optimality}, one sees that if $|\langle \a_j, \v^*\rangle | < 1$ for all $\a_j \in \cA_-$ then the vector $\c^*$ is  subspace-preserving.
	This suggests that the condition for $\BP(\A, \b)$ to be subspace-preserving depends on the dot product between the data points and the solution to the dual optimization problem.
	This motivates the following result, which gives a tight condition for solutions in $\BP(\A, \b)$ to be subspace-preserving.

	\begin{theorem}[Tight Instance Dual Condition (T-IDC)]\label{thm:T-IDC}
		Let $\0 \neq \b \in \cS_0$. 
		All elements in $\BP(\A, \b)$ are subspace-preserving if and only if  
		\begin{equation}\label{eq:T-IDC}
		\exists \v \in \BP_{\Dual}(\A_0, \b)
		\ \text{satisfying} \ 		
		\| \A_-^\transpose \v \|_\infty < 1.
		\end{equation}
	\end{theorem}
		The condition in \cref{eq:T-IDC} requires that there exists a point from the set of dual solutions $\BP_{\Dual}(\A_0,\b)$ such that the absolute value of the dot product between the dual solution and all points in $\cA_-$ is less than one. 
		The interpretation of this condition will become clearer once we understand the geometric structure of the set $\BP_{\Dual}(\A_0,\b)$. 
		To that end, observe from the definition of $\BP_{\Dual}(\A_0,\b)$ and the fact that both $\b$ and all columns of $\A_0$ lie in the subspace $\cS_0$, that if $\v\in\BP_{\Dual}(\A_0,\b)$ then $\v + \cS_0^\perp \subseteq \BP_{\Dual}(\A_0, \b)$, where $\cS_0^\perp$ is the space orthogonal to $\cS_0$.
		Consequently, $\BP_{\Dual}(\A_0, \b)$ is composed of a collection of affine subspaces, each of dimension $D - d_0$, that are perpendicular to $\cS_0$. 
		Among all points in $\BP_{\Dual}(\A_0, \b)$, of particular interests are those that lie in the intersection of $\BP_{\Dual}(\A_0, \b)$ and $\cS_0$, which we define as the set of dual points as follows.
	
	
	\begin{definition}[dual points associated with $\A_0$ and $\b$]\label{def:instance-dual-point}
		Given $\cA_0 \subseteq \cS_0$ and $\0 \ne \b \in \cS_0$, we define the set of dual points, denoted by $\cD(\A_0, \b)$, as 
		\begin{equation}\label{eq:instance-dual-point}
		\cD(\A_0,\b)
		:= \Argmax_{\v\in\Re^N} \ \langle \v, \b \rangle \st \ \|\A_0 ^\transpose \v\|_{\infty} \le 1, \ \v \in \cS_0.
		\end{equation}
	\end{definition}
	
	The relationship between $\BP_{\Dual}(\A_0,\b)$ and $\cD(\A_0,\b)$ is now made precise.
	\begin{lemma}\label{lem:B-shifted-B}
		For any $\cA_0 \subseteq \cS_0$ and $\0 \ne \b \in \cS_0$, we have
		\begin{equation}\label{eq:dual-solution-dual-point}
		\BP_{\Dual}(\A_0, \b) = \cD(\A_0, \b) + \cS_0^\perp.
		\end{equation}
	\end{lemma}
	\begin{proof}
		First, we show that the left-side of \cref{eq:dual-solution-dual-point} is contained in the right-hand side.  To that end, let $\bar\v\in\BP_{\Dual}(\A_0, \b)$.  Then, define $\bar \p$ as the projection of $\bar\v$ onto $\cS_0$, and then define $\bar\q = \bar\v - \bar \p \in \cS_0^\perp$ so that $\bar \v = \bar \p + \bar \q$.  Thus, it remains to prove that $\bar \p \in  \cD(\A_0, \b)$. Using $\b\in\cS_0$, $\bar\q\in\cS_0^\perp$, and $\bar\v\in\BP_{\Dual}(\A_0, \b)$ it follows that
		$$
		\bar\p^\transpose \b
		= (\bar\v - \bar \q)^\transpose \b
		= \bar\v^\transpose \b
		\geq \v^\transpose \b
		\ \text{for all $\v$ satisfying $\|\A_0^\transpose\v\|_\infty \leq 1$.}
		$$
		Combining this with $\bar\p\in\cS_0$ and
		$\|\A_0^\transpose \bar\p\|_\infty = \|\A_0^\transpose \bar\v \|_\infty \leq 1$ shows that $\bar\p \in  \cD(\A_0, \b)$.

		Next, we show that the right-side of \cref{eq:dual-solution-dual-point} is contained in the left-hand side. 
		Let $\bar\p\in\cD(\A_0, \b)$ and $\bar\q\in\cS_0^\perp$.  By defining $\bar \v = \bar \p + \bar \q$, it remains to prove that $\bar\v \in \BP_{\Dual}(\A_0, \b)$. For a proof by contradiction, suppose it was not true. Since $\|\A_0^\transpose \bar \v\|_\infty = \|\A_0^\transpose \bar \p\|_\infty \leq 1$ (i.e., $\bar\v$ is feasible for $\BP_{\Dual}(\A_0, \b)$), we know there exists a vector $\hat \v$ satisfying $\hat\v^\transpose \b > \bar \v^\transpose \b$ and $\|\A_0^\transpose\hat\v\|_\infty \leq 1$. Defining $\hat \p$ as the projection of $\hat\v$ onto $\cS_0$ and $\hat\q = \hat\v-\hat\p\in\cS_0^\perp$, it follows using $\b\in\cS_0$, $\hat\q\in\cS_0^\perp$, and $\bar\q\in\cS_0^\perp$ that
		$$
		\hat\p^\transpose \b
		= \hat\v^\transpose \b
		>  \bar\v^\transpose \b
		= \bar\p^\transpose \b
		\ \ \text{and} \ \ 
		\|\A_0^\transpose \hat \p\|_\infty
		= \|\A_0^\transpose \hat \v\|_\infty
		\leq 1,
		$$
		which, together with $\hat\p\in\cS_0$, contradicts $\bar\p\in\cD(\A_0, \b)$. This completes the proof.
		%
	\end{proof}
	

	Since the optimization problem in \cref{eq:instance-dual-point} is a linear program with constraint set $\cK_0^o$, the solution set $\cD(\A_0, \b)$ is a face of the convex polyhedron $\cK_0^o$ (see~\cite{Zhang:JCM13}). 
	If $\b$ is drawn from a distribution that is independent of $\cA_0$, then the optimal face is $0$-dimensional (i.e., a vertex of $\cK_0^o$) with probability one, in which case $\cD(\A_0, \b)$ contains a single point.
	Correspondingly, the set of dual solutions $\BP_{\Dual}(\A_0, \b)$ is a $D-d_0$ dimensional affine subspace that passes through the unique point in $\cD(\A_0, \b)$ and is perpendicular to $\cS_0$ (see \cref{fig:geometry-dual}). 
	The optimal face may also be higher-dimensional when $\b$ is perpendicular to such a face, in which case $\cD(\A_0, \b)$ contains infinitely many points and $\BP_{\Dual}(\A_0, \b)$ contains a union of infinitely many affine subspaces. 
		
	We end the discussion of the tight condition for instance recovery by BP with a proof of \cref{thm:T-IDC}.

	\begin{proof}[Proof of \cref{thm:T-IDC}]
		Without loss of generality, let $\A = [\A_0 \  \A_-]$. It will be convenient to define $\c^* = (\c_0^*, \0)$, where $\c_0^* \in \BP(\A_0,\b)$
		which exists since $\b \in \cS_0 = \spann(\A_0)$. 
		
		For the ``if'' direction, let $\v^* \in \BP_{\Dual}(\A_0, \b)$ be any point satisfying \cref{eq:T-IDC} and $\bar{\c} = (\bar{\c}_0, \bar{\c}_-)$ any point in $\BP(\A, \b)$. 
		Since the linear problems 
		$\BP(\A_0,\b)$ and $\BP_{\Dual}(\A_0,\b)$ are dual, strong duality, $\|\A_0^\transpose \v^*\|_\infty \leq 1$, and $\|\A_-^\transpose \v^*\|_\infty < 1$ yield
		\begin{align*}
			\|\bar{\c}\|_1
			&\le\|\c^*\|_1 = \langle \v^*, \b \rangle = \langle \v^*, \A \bar{\c} \rangle
			= \langle \v^*, \A_0 \bar{\c}_0 \rangle + \langle \v^*, \A_- \bar{\c}_- \rangle \\
			&= \langle \A_0^\transpose \v^*, \bar{\c}_0 \rangle + \langle \A_-^\transpose\v^*, \bar{\c}_- \rangle 
			\le \|\A_0^\transpose \v^*\|_\infty  \|\bar{\c}_0\|_1 + \|\A_-^\transpose \v^*\|_\infty  \|\bar{\c}_-\|_1 \le \|\bar{\c}\|_1,
		\end{align*}
		so that all these inequalities are equalities.
		Combining the last of these inequalities (now as an equality) with $\|\A_0^\transpose \v^*\|_\infty \leq 1$ and $\|\A_-^\transpose \v^*\|_\infty < 1$ shows that
		$$
		\|\bar{\c}_-\|_1
		=  \frac{\big( \|\A_0^\transpose\v^* \|_\infty - 1 \big) \|\bar{\c}_0\|_1}{1-\|\A_-^\transpose \v^*\|_\infty}
		\leq 0,
		$$
		implying that $\bar{\c}_- = \0$; thus, $\bar{\c}$ is subspace-preserving, as claimed.
		
		Next, to prove the ``only if'' direction, we suppose that all vectors in $\BP(\A, \b)$ are subspace-preserving and construct a $\v^* \in \BP_{\Dual}(\A_0, \b)$ such that $\|\A_-^\transpose \v^*\|_\infty < 1$. 
		
		The support of $\c^*$ is $T=\{j: c^*_j \ne 0\}$ and define $R = \{j: \a_j \in \cA_0\} \setminus T$ so that $\cA_0 \equiv \{\a_j\}_{j\in T\cup R}$.
		Following the proof in \cite{Zhang:ACM16}, let us define
		\begin{equation}\label{eq:prf-construct-v}
		\v^* = \argmin_{\v\in\Re^D} \|\A_-^\transpose \v\|_\infty \st \A_T^\transpose \v = \sgn(\c_T^*), ~~ \|\A_R^\transpose \v\|_\infty \le 1,
		\end{equation}
		where $\A_T$ and $\A_R$ denote submatrices of $\A$ containing the columns indexed by $T$ and $R$, respectively, and $\c^*_T$ denotes the subvector of $\c^*$ that contains entries indexed by $T$.  
		The optimality conditions satisfied by $\c^*_0 \in \BP(\A_0,\b)$ 
		show that problem \cref{eq:prf-construct-v} is feasible. Moreover, using $\c^*_0\in\BP(\A_0,\b)$, the definitions of $R$ and $T$, and $\c^*_R = 0$ it follows that any feasible point $\v$ for problem \cref{eq:prf-construct-v} satisfies
		\begin{align*}
			\v^\transpose \b
			&= \v^\transpose \A_0 \c^*_0
			= (\A_0^\transpose \v)^\transpose \c^*_0 \\
			&= (\A_T^\transpose \v)^\transpose \c^*_T + (\A_R^\transpose \v)^\transpose \c^*_R
			= (\sgn(\c^*_T))^\transpose \c^*_T
			= \|\c^*_T\|_1
			= \|\c^*_0\|_1.
		\end{align*}
		It follows from this equality, the fact that $\|\c^*_0\|_1$ is the optimal value for the solutions in $\BP(\A_0,\b)$, and strong duality between $\BP(\A_0,\b)$ and $\BP_{\Dual}(\A_0,\b)$ that any feasible point for problem \cref{eq:prf-construct-v} belongs to $\BP_{\Dual}(\A_0,\b)$. Thus, $\v^* \in\BP_{\Dual}(\A_0,\b)$.	 
		It remains to prove that the objective value of \cref{eq:prf-construct-v} at $\v^*$ is strictly less than one. 
		
		Defining $\u = \A^\transpose \v$, the optimization problem \cref{eq:prf-construct-v} is equivalent to the problem
		\begin{equation}\label{eq:prf-construct-u}
		\min_{\u\in\Re^N} \|\u_-\|_\infty \st \u_T = \sgn(\c_T^*), ~~ \|\u_R\|_\infty \le 1, ~~\u \in \range(\A^\transpose).
		\end{equation}	
		Note that $\u \in \range(\A^\transpose)$ if and only if $\Q^\transpose \u = \0$, where $\Q$ denotes a basis of the null space of $\A$. 
		By writing $\u = \w + \sgn(\c^*)$, we can equivalently write \cref{eq:prf-construct-u} as
		\begin{equation}\label{eq:prf-construct-w}
		\min_{\w\in\Re^N} \|\w_-\|_\infty \st \w_T = \0, ~~ \|\w_R\|_\infty \le 1, ~~ \Q^\transpose \w = - \Q^\transpose \sgn(\c^*).
		\end{equation}
		The dual of this problem may be derived to be
		\begin{equation}\label{eq:prf-construct-w-dual}
		\max_{\h\in\Re^N} \ \langle \h, \sgn(\c^*)\rangle - \|\h_R\|_1 
		\st \|\h_-\|_1 \le 1, ~~ \h\in \ns(\A).
		\end{equation}
		Since \cref{eq:prf-construct-w} is a feasible linear program 
		(it is equivalent to \cref{eq:prf-construct-v}, which is feasible),
		the optimal objective values of problems \cref{eq:prf-construct-w} and \cref{eq:prf-construct-w-dual} are finite and equal. Combining this with the fact that the objective values in  \cref{eq:prf-construct-v} and  \cref{eq:prf-construct-w} are equal under the relation $\A^\transpose \v^* = \w^* + \sgn(\c^*)$, we have that any solution $\h^*$ to \cref{eq:prf-construct-w-dual} satisfies 
		\begin{equation} \label{eq:bd-h}
		\|\A_-^\transpose \v^*\|_\infty 
		= \|\w^*_- + \sgn(\c^*_-)\|_\infty
		= \|\w^*_-\|_\infty = \langle\h^*,\sgn(\c^*)\rangle - \|\h_R^*\|_1.
		\end{equation}
		Thus, it remains to prove that the right-hand side of \cref{eq:bd-h} is less than one.
		
		
		First, consider the case when $\h^*_- = \0$.  We claim that the objective value of \cref{eq:prf-construct-w-dual} at $\h^*$ (i.e., the right-hand side of \cref{eq:bd-h}) must be zero. 
		To see this, note that if the objective value of \cref{eq:prf-construct-w-dual} is not zero, then it must be strictly positive since the objective of \cref{eq:prf-construct-w} is nonnegative and strong duality holds. 
		Thus, $2 \h^*$ is a feasible solution to \cref{eq:prf-construct-w-dual} that gives a larger objective value, which contradicts the optimality of $\h^*$.
		
		Second, consider the case when $\h^*_- \ne \0$.
		Choose any $t < 0$ small enough that $\sgn(\c^*_T) = \sgn(\c_T^* + t\h^*_T)$.
		Since all elements in $\BP(\A, \b)$ are subspace-preserving by assumption, 
		it follows that $\c^* \in \BP(\A, \b)$. 
		Moreover, since $\c^* + t\h^*$ is feasible for \cref{eq:def-BP} (recall that $\h^*\in\ns(\A)$ according to \cref{eq:prf-construct-w-dual}) but is not subspace-preserving because $\h^*_- \neq 0$, we know that $\c^* + t\h^*$ is not optimal for \cref{eq:def-BP}.
		Therefore, we have
		\begin{equation}
		\begin{split}
		\|\c^*\|_1 
		&< \|\c^* + t \h^*\|_1 
		= \|\c^*_T + t \h^*_T\|_1 + |t|\|\h^*_R\|_1 + |t|\|\h^*_-\|_1 \\
		&= \langle \c_T^* + t \h^*_T, \sgn(\c_T^*)\rangle + |t|(\|\h^*_{R}\|_1 + \|\h^*_-\|_1)\\
		&= \|\c^*\|_1 + t \langle \h^*, \sgn(\c^*)\rangle - t(\|\h^*_{R}\|_1 + \|\h^*_-\|_1),
		\end{split}
		\end{equation}
		which may be combined with $t < 0$ and the constraint of \cref{eq:prf-construct-w-dual} to conclude that		
		$$
		\langle \h^*, \sgn(\c^*)\rangle - \|\h^*_R\|_1
		< \|\h^*_-\|_1 \le 1,
		$$
		which establishes that the right-hand side of \cref{eq:bd-h} is less than one, as claimed.
	\end{proof}
	
	\subsubsection{A sufficient condition}\label{sec.suff.bp}
	
	\cref{thm:T-IDC} shows that BP produces subspace-preserving solutions if and only if there exists a dual solution in $\BP_{\Dual}(\A_0, \b)$ such that the absolute value of the dot product between the dual solution and all points in $\cA_-$ is less than one. 
	Let us note that condition~\cref{eq:T-IDC} in \cref{thm:T-IDC} is equivalent to
	\begin{equation}\label{eq:T-IDC-reformulate}
	\exists \v \in \BP_{\Dual}(\A_0, \b)
	\ \ \text{satisfying} \ \ 
	\max_{\a \in \cA_-} \left|\left\langle \frac{\v}{\|\v\|_2}, \a \right\rangle\right| 
	< \frac{1}{\|\v\|_2}.
	\end{equation}
		Here, 
		the left-hand side of the inequality in \cref{eq:T-IDC-reformulate} is the maximum absolute value of the dot product between the  normalized vector $\v$ and any point from $\cA_-$, which is small when the points from these two sets are sufficiently separated. 
		%
		On the other hand, a finite upper bound on $\|\v\|_2$ (i.e., a positive lower bound on the right-hand side of \cref{eq:T-IDC-reformulate}) does not exist in general since the set $\BP_{\Dual}(\A_0, \b)$ is unbounded when $d_0 < D$ (see \cref{lem:B-shifted-B}). 
		Nonetheless, the smallest $\ell_2$-norm solutions in $\BP_{\Dual}(\A_0, \b)$ are those that lie in $\cS_0$, which are exactly the dual points in $\cD(\A_0, \b)$.
		This motivates us to introduce the following weaker subspace-preserving recovery condition.
	
	
	\begin{theorem}[Instance Dual Condition (IDC)]\label{thm:IDC}
		Let $\0 \neq \b \in \cS_0$. 
		All elements in $\BP(\A, \b)$ are subspace-preserving if 
		\begin{equation}\label{eq:IDC}
		\exists \v \in \cD(\A_0, \b)
		\ \text{satisfying} \ 		
		\| \A_-^\transpose \v \|_\infty < 1.
		\end{equation}
	\end{theorem}
	\begin{proof}
		Let $\v$ be the vector satisfying condition \cref{eq:IDC}.  It follows from \cref{eq:IDC} and \cref{lem:B-shifted-B} that $\v\in\BP_{\Dual}(\A_0,\b)$ and $\| \A_-^\transpose \v \|_\infty < 1$.    Combining this result with \cref{thm:T-IDC} shows that all elements in $\BP(\A, \b)$ are subspace-preserving.
	\end{proof}
	The condition in \cref{eq:IDC} requires that there exists a dual point $\v$ that has small inner product with all points in $\cA_-$. 
	In the next section, we present a weaker sufficient condition for subspace-preserving recovery by using an upper bound for $\v \in \cD(\A_0, \b)$. 
	
	
	\subsubsection{A geometrically intuitive sufficient condition}\label{sec.geom.bp}

	We derive a geometrically intuitive sufficient condition that is characterized by the distribution of points in $\cA_0$ and the angular distance between the dual set $\cD(\A_0, \b)$ and $\cA_-$. 
	Note that if $\v \in \cD(\A_0,\b) \subseteq\cK_0^o$, it follows from \cref{def:circumradius} and \cref{thm:inradius-coveringradius-circumradius} that
	\begin{equation}\label{eq:prf-bound-dual}
	\|\v\|_2 \le R_0 = 1 / r_0.
	\end{equation}
	Combining this result with \cref{eq:T-IDC-reformulate} gives the following theorem. 
	\begin{theorem}[Geometric Instance Dual Condition (G-IDC)]\label{thm:G-IDC}
		Let $\0 \ne \b \in \cS_0$. All elements in the set $\BP(\A, \b)$ are subspace-preserving if 
		\begin{equation}\label{eq:G-IDC}
		\exists \v \in \cD(\A_0, \b)
		\ \ \text{satisfying} \ \
		r_0 > \mu(\v, \cA_-),
		\end{equation}
		with \cref{eq:G-IDC} holding if and only if
		\begin{equation}\label{def:angular}
		\exists \v \in \cD(\A_0, \b) \ \ \text{satisfying} \ \ 
		\gamma_0 < \theta(\{\pm\v\}, \cA_-).
		\end{equation}
	\end{theorem}
	
	\begin{proof}
		Let $\v$ 
		satisfy \cref{eq:G-IDC}. It follows from \cref{eq:G-IDC} and \cref{eq:prf-bound-dual} that
		\begin{equation}
		\left|\left\langle \frac{\v}{\|\v\|_2}, \a \right\rangle\right|
		< r_0
		\leq \frac{1}{\|\v\|_2} 
		\ \ \text{for all $\a \in \cA_-$,}
		\end{equation}
		which is equivalently to
		$\| \A_-^\transpose \v\|_\infty < 1$.
		From \cref{eq:dual-solution-dual-point} we have 
		$\v \in \BP_{\Dual}(\A_0, \b)$. Combining this fact with $\| \A_-^\transpose \v\|_\infty < 1$ and 
		\cref{thm:T-IDC} shows that all elements in $\BP(\A,\b)$ are subspace-preserving, as claimed.  The fact that \cref{eq:G-IDC} and \cref{def:angular} are equivalent follows from \cref{thm:inradius-coveringradius-circumradius}. 
	\end{proof}
	
	For condition \cref{eq:G-IDC} to hold, the points in $\cA_0$ must be well distributed within $\cS_0$ as measured by the inradius $r_0$, and there must exist a dual point $\v \in \cD(\A_0, \b)$ that is well separated from $\cA_-$ as measured by the inner products.
	The geometric interpretation of \cref{eq:G-IDC} is clearer when written in 
	its equivalent angular form \cref{def:angular}.
	The inequality \cref{def:angular} states that all points from $\cA_-$ need to be outside of a pair of antipodal spherical caps centered at $\pm \v$ with radius equal to the covering radius $\gamma_0$ {(see~\cref{fig:geometry-G-IDC})}.

	Prior results in \cite{Soltanolkotabi:AS12} are related to \cref{thm:G-IDC}.  
	In \cite{Soltanolkotabi:AS12}, a dual point is defined as the minimum $\ell_2$-norm solution of \cref{eq:instance-dual-point}, and their instance recovery condition for BP is that \cref{eq:G-IDC} is satisfied for this dual point.
	When $\cD(\A_0, \b)$ contains a single point, the recovery condition from \cite{Soltanolkotabi:AS12} is equivalent to the condition in \cref{eq:G-IDC}. 
	Otherwise, the recovery condition from \cite{Soltanolkotabi:AS12} may be stronger than our condition in \cref{eq:G-IDC}. 
	
	\subsection{Instance recovery conditions for OMP}
	\label{sec:instance_recovery_OMP}
	
	In this section we consider the instance recovery problem for OMP.  In \cref{sec.suff.omp} we give a sufficient condition for ensuring subspace-preserving recovery, and in \cref{sec.geom.omp} we give a weaker sufficient condition for ensuring recovery that has a clear geometric interpretation.
	
	\subsubsection{A sufficient condition}\label{sec.suff.omp}
	
	The OMP method computes a sparse representation using \cref{alg.omp}.
	One way to derive conditions that guarantee instance recovery is to consider the performance of $\OMP$ on a related problem, namely with data $\A_0$ and $\b$.  A  feature of this related problem is that the solution returned by $\OMP$ will be subspace-preserving since only $\A_0$ is used. When \cref{alg.omp} is called with input data $\A_0$ and $\b$, we denote the $k$th computed working set, representation vector, and residual vector, respectively, by $\cW^{(k)}_0$, $\c^{(k)}_0$, and $\v^{(k)}_0$. Motivated by (\ref{eq:OMP-maximizer}) in \cref{alg.omp},  conditions guaranteeing subspace preservation when the entire data matrix $\A$ is used can then be written in terms of the residual vector $\v^{(k)}_0$ and its angle with the columns of $\A_0$ and $\A_-$. 
	%
	%
	%
	%
	This observation motivates the following definition.

	\ifdefined\draft
	\else
	
	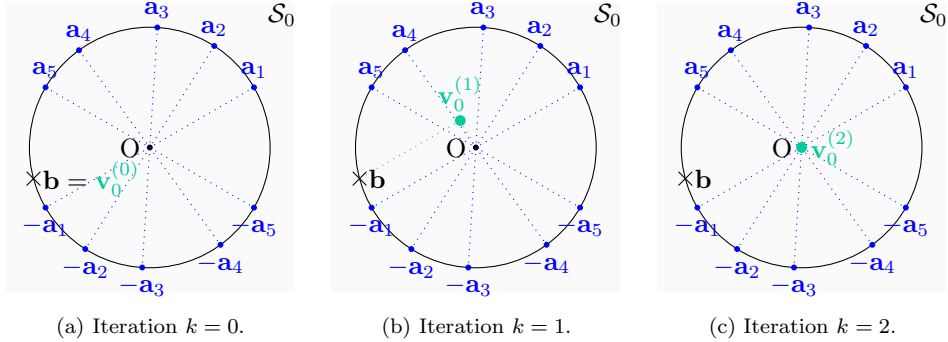
\begin{figure*}[t]
		\centering
		\def\nPoint{5}
		\def\primalAngle
		{{30.00, 57.60, 86.40, 126.00, 150.00, 210.00}}
		\def\dualRadius
		{{1.030, 1.032, 1.063, 1.022, 1.154, 1.030}}
		\def\dualAngle
		{{43.80, 72.00, 106.20, 138.00, 180.00, 223.80}}
		\def\instancePoint{195}
		\def\instanceResidualAngle{120}
		\def\instanceResidualRadius{0.2588} 
		\subfloat[Iteration $k=0$.\label{fig:geometry-Residual-0}]
		{
			\begin{tikzpicture}[scale = 1.6]
			\coordinate (0) at (0,0);
			\def\wAngle{0};
			\def\wRefAngle{\primalAngle[0]};
			
			\def\Sx{1.2}         
			\def\Sy{1.2}
			\filldraw[draw=none,fill=gray!20, opacity=0.2] 
			(-\Sx,-\Sy) -- (-\Sx, \Sy) -- (\Sx, \Sy) -- (\Sx, -\Sy) -- cycle;
			\node[black] at (1.1, 1.1, 0) {$\cS_0$};
			\draw[black, fill = none] (0) circle [radius = 1]; 
			\node[left, black] at (0) {O}; 
			\draw [black, fill=black] (0) circle [radius=0.02];
			\foreach \i in {1, 2, ..., \nPoint}
			{
				\def\angle{\primalAngle[\i-1]}
				\draw [pPntColor, fill=pPntColor] (cos \angle, sin \angle) circle [radius=0.02];
				\draw [pPntColor, fill=pPntColor] (-cos \angle, -sin \angle) circle [radius=0.02];
				\draw [dotted, pPntColor] (cos \angle, sin \angle) -- (-cos \angle, -sin \angle);
				\node[above, pPntColor] at (cos \angle, sin \angle) {$\a_\i$}; 
				\node[below, pPntColor] at (-cos \angle, -sin \angle) {$-\a_\i$}; 
			}
			\draw (cos \instancePoint, sin \instancePoint) node[cross] {};
			\node[right, black] at (cos \instancePoint, sin \instancePoint) {$\b = {\color{iPntColor}\v_0^{(0)}}$}; 
			\end{tikzpicture}
		}
		~
		\subfloat[Iteration $k=1$.\label{fig:geometry-Residual-1}]
		{
			\begin{tikzpicture}[scale = 1.6]
			\coordinate (0) at (0,0);
			\def\wAngle{0};
			\def\wRefAngle{\primalAngle[0]};
			
			\def\Sx{1.2}         
			\def\Sy{1.2}
			\filldraw[draw=none,fill=gray!20, opacity=0.2] 
			(-\Sx,-\Sy) -- (-\Sx, \Sy) -- (\Sx, \Sy) -- (\Sx, -\Sy) -- cycle;
			\node[black] at (1.1, 1.1, 0) {$\cS_0$};
			\draw[black, fill = none] (0) circle [radius = 1]; 
			\node[left, black] at (0) {O}; 
			\draw [black, fill=black] (0) circle [radius=0.02];
			\foreach \i in {1, 2, ..., \nPoint}
			{
				\def\angle{\primalAngle[\i-1]}
				\draw [pPntColor, fill=pPntColor] (cos \angle, sin \angle) circle [radius=0.02];
				\draw [pPntColor, fill=pPntColor] (-cos \angle, -sin \angle) circle [radius=0.02];
				\draw [dotted, pPntColor] (cos \angle, sin \angle) -- (-cos \angle, -sin \angle);
				\node[above, pPntColor] at (cos \angle, sin \angle) {$\a_\i$}; 
				\node[below, pPntColor] at (-cos \angle, -sin \angle) {$-\a_\i$}; 
			}
			\draw (cos \instancePoint, sin \instancePoint) node[cross] {};
			\node[right, black] at (cos \instancePoint, sin \instancePoint) {$\b$}; 
			\draw [iPntColor, fill=iPntColor] (\instanceResidualRadius * cos \instanceResidualAngle, \instanceResidualRadius * sin \instanceResidualAngle) circle [radius=0.04];
			\node[above, iPntColor] at (\instanceResidualRadius * cos \instanceResidualAngle, \instanceResidualRadius * sin \instanceResidualAngle) {$\v_0^{(1)}$}; 
			\draw [dotted, iPntColor] (\instanceResidualRadius * cos \instanceResidualAngle, \instanceResidualRadius * sin \instanceResidualAngle) -- (cos \instancePoint, sin \instancePoint);  
			\draw [dotted, iPntColor] (\instanceResidualRadius * cos \instanceResidualAngle, \instanceResidualRadius * sin \instanceResidualAngle) -- (0, 0);  
			\end{tikzpicture}
		}
		~
		\subfloat[Iteration $k=2$.\label{fig:geometry-Residual-2}]
		{
			\begin{tikzpicture}[scale = 1.6]
			\coordinate (0) at (0,0);
			\def\wAngle{0};
			\def\wRefAngle{\primalAngle[0]};
			
			\def\Sx{1.2}         
			\def\Sy{1.2}
			\filldraw[draw=none,fill=gray!20, opacity=0.2] 
			(-\Sx,-\Sy) -- (-\Sx, \Sy) -- (\Sx, \Sy) -- (\Sx, -\Sy) -- cycle;
			\node[black] at (1.1, 1.1, 0) {$\cS_0$};
			\draw[black, fill = none] (0) circle [radius = 1]; 
			\node[left, black] at (0) {O}; 
			\draw [black, fill=black] (0) circle [radius=0.02];
			\foreach \i in {1, 2, ..., \nPoint}
			{
				\def\angle{\primalAngle[\i-1]}
				\draw [pPntColor, fill=pPntColor] (cos \angle, sin \angle) circle [radius=0.02];
				\draw [pPntColor, fill=pPntColor] (-cos \angle, -sin \angle) circle [radius=0.02];
				\draw [dotted, pPntColor] (cos \angle, sin \angle) -- (-cos \angle, -sin \angle);
				\node[above, pPntColor] at (cos \angle, sin \angle) {$\a_\i$}; 
				\node[below, pPntColor] at (-cos \angle, -sin \angle) {$-\a_\i$}; 
			}
			\draw (cos \instancePoint, sin \instancePoint) node[cross] {};
			\node[right, black] at (cos \instancePoint, sin \instancePoint) {$\b$}; 
			\draw [iPntColor, fill=iPntColor] (0, 0) circle [radius=0.04];
			\node[right, iPntColor] at (0, 0) {$\v_0^{(2)}$}; 
			\end{tikzpicture}
		}
		\caption[Illustration of the set of residual points]{\label{fig:geometry-residual}Illustration of the set of residual points $\cR(\A_0, \b)$ (see \cref{def:residual-points}). 
			The atoms $\cA_0:=\{\a_j\}_{j=1}^5$ (drawn in {\color{pPntColor}blue}) and $\b$ (marked as $\times$) lie on the unit circle of a two-dimensional subspace $\cS_0$. 
			We show the residual vectors $\v_0^{(k)}$ (drawn in {\color{iPntColor}green}) computed in iterations $k \in \{0,1,2\}$ of $\OMP(\A_0, \b)$.
			(a) $k=0$: $\cW_0^{(0)}=\emptyset$ and $\v_0^{(0)}=\b$. 
			(b) $k=1$: $\cW_0^{(1)}=\{\a_1\}$ and $\v_0^{(1)}$ is the component of $\b$ that is perpendicular to $\a_1$.
			(c) $k=2$: $\cW_0^{(2)}=\{\a_1, \a_4\}$ and $\v_0^{(2)} = \0$. 
			By definition, $\cR(\A_0, \b) = \{\v_0^{(0)}, \v_0^{(1)}\}$.
		}
	\end{figure*}
	
	\fi
	
	\begin{definition}[residual points]\label{def:residual-points}
		Given any $\cA_0 \subseteq \cS_0$ and any $\0 \neq \b \in \cS_0$, we define the set of residual points, denoted by $\cR(\A_0, \b)$, as the set of nonzero residual vectors computed by \cref{alg.omp} with input data $\A_0$ and $\b$.
	\end{definition}
	An illustration of the residual points $\cR(\A_0, \b)$ is given in \cref{fig:geometry-residual}.
	The size of $\cR(\A_0, \b)$ is equal to the number of iterations computed by $\OMP(\A_0, \b)$, which is at most $d_0$. 
	In addition, $\cR(\A_0, \b)$ is a subset of $\cS_0$ since each residual vector is computed by subtracting a linear combination of vectors in $\cA_0$ from the vector $\b$, where both $\b$ and all points in $\cA_0$ lie in the subspace $\cS_0$. 
	
	\cref{def:residual-points} may now be used to provide a sufficient condition for $\OMP$ to produce a subspace-preserving representation, as we now state.
	
	
	\begin{theorem}[Instance Residual Condition (IRC)]\label{thm:IRC}
		Given 
		$\0 \neq \b\in \cS_0$, the solution 
		$\OMP(\A, \b)$ is subspace-preserving if 
		\begin{equation}\label{eq:IRC}
		\forall \v \in \cR(\A_0, \b)
		\ \ \text{it holds that} \ \
		\|\A_-^\transpose \v\|_\infty < \|\A_0^\transpose \v\|_\infty.
		\end{equation}
	\end{theorem}
	\begin{proof}
		We prove that the sequence of working sets $\{\cW^{(k)}\}$ computed by $\OMP$ with input $\A$ and $\b$, 
		and the sequence of working sets $\{\cW^{(k)}_0\}$ computed by $\OMP$ with input $\A_0$ and $\b$ 
		are the same. 
		Once this is established, and observing that  $\cW^{(k)} = \cW^{(k)}_0 \subseteq \cA_0$, it immediately follows that 
		$\OMP(\A,\b)$ 
		is subspace-preserving. 
		
		To prove the working sets are equal using a proof by contradiction, let $k$ be the first iteration so that $\cW^{(k)} = \cW^{(k)}_0$ and $\cW^{(k+1)} \neq \cW^{(k+1)}_0$. Since $k$ is the smallest iteration when the working sets change, it follows from \eqref{eq:OMP-residual} that $\c^{(k)} = \c^{(k)}_0$, which in turn implies that 
			$\v^{(k)} = \v^{(k)}_0$. 
			Combining this with \cref{eq:OMP-maximizer}, $\cW^{(k)} = \cW^{(k)}_0$, and $\cW^{(k+1)} \neq \cW^{(k+1)}_0$ shows that
		$\| \A^\transpose_- \v^{(k)}_0\|_\infty \ge  \|\A_0^\transpose \v^{(k)}_0\|_\infty$, which contradicts \cref{eq:IRC}.
	\end{proof}
	
	Condition \cref{eq:IRC} is actually both necessary and sufficient for OMP  
	to select a point from $\cA_0$ during \emph{every} iteration (assuming the optimization problem in \cref{eq:OMP-maximizer} has a unique optimizer). 
	On the other hand, 	
	condition \cref{eq:IRC} is only sufficient for subspace-preserving recovery. 
	Cases exist when $\OMP$ with input data $\A$ and $\b$ selects points from $\A_-$ 
	during iterations prior to termination, but the representation computed by the final iteration is subspace-preserving.
	For example, let 
	$\cA=\cA_0 \cup \cA_- \subseteq \Re^3$ with
	\begin{equation}\label{eq:OMP-example}
	\cA_0=\{(\cos 3^\circ, \sin 3^\circ, 0), (\cos 2^\circ, -\sin 2^\circ, 0) \}
	\ \ \text{and} \ \ 
	\cA_- = \{(\cos 1^\circ, 0, \sin 1^\circ)\}.
	\end{equation}
	Note that $\cA_0 \subseteq \cS_0$ where $\cS_0$ is the $x$-$y$ plane. Let $\b = (1, 0, 0) \in \cS_0$. Condition~\cref{eq:IRC} does not hold because the first iteration of OMP selects the point in $\cA_-$. 
	Nonetheless, the solution $\OMP(\A, \b)$ is subspace-preserving because  
	it terminates after three iterations with working set $\cW^{(3)} =  \cA_0\cup \cA_-$ and subspace-preserving solution $\c^{(3)}$.
	
	\subsubsection{A geometrically intuitive sufficient condition}\label{sec.geom.omp}
	
	
	The aim is now to derive a geometrically intuitive sufficient condition characterized by the distribution of points in $\cA_0$ and the angular distance between the residual set $\cR(\A_0, \b)$ and $\cA_-$. 
	
	
	\begin{theorem}[Geometric Instance Residual Condition (G-IRC)]\label{thm:G-IRC}
		Given $\0 \neq \b \in\cS_0$, 
		the solution $\OMP(\A, \b)$ is a subspace-preserving representation if 
		\begin{equation}\label{eq:G-IRC}
		\forall \v \in \cR(\A_0, \b)
		\ \text{it holds that} \
		r_0 > \mu(\v, \cA_-),
		\end{equation}
		with \cref{eq:G-IRC} holding if and only if
		\begin{equation} \label{def:angular.omp}
		\forall \v \in \cR(\A_0, \b)
		\ \text{it holds that} \ 
		\gamma_0 < \theta(\{\pm\v\}, \cA_-).
		\end{equation}
	\end{theorem}
	\begin{proof}
		We prove that condition \cref{def:angular.omp} holding implies that condition \cref{eq:IRC} holds. To that end, let $\v \in \cR(\A_0, \b)$. 
		From the definition of the covering radius it holds that $\gamma_0= \max \{ \theta(\w, \pm\cA_0): \w\in \cS_0 \cap \Sp^{D-1} \}$, and since $\v/\|\v\|_2 \in \cS_0 \cap \Sp^{D-1}$ it follows that $\gamma_0 \ge \theta(\v/\|\v\|_2, \pm\cA_0) = \theta(\{\pm\v\}, \cA_0)$. 
		Combining this with \cref{def:angular.omp}, we get
		\begin{equation}
		\theta(\{\pm\v\}, \cA_-) > \theta(\{\pm\v\}, \cA_0).
		\end{equation}
		By taking cosine of both sides we get \cref{eq:IRC}, which with \cref{thm:IRC} shows that $\OMP(\A,\b)$ is subspace-preserving, as claimed. The fact that \cref{eq:G-IRC} and \cref{def:angular.omp} are equivalent follows from \cref{thm:inradius-coveringradius-circumradius}.
	\end{proof}

	Condition \cref{eq:G-IRC} requires that the points $\cA_0$ are well enough distributed as measured by the inradius of $\pm\cA_0$, and the residual points to be separated from points in $\cA_-$ as measured by their coherence.
	On the other hand, the equivalent condition \cref{def:angular.omp} states that the points $\cA_-$ do not lie in the spherical caps of radius $\gamma_0$ centered at the residual points in $\cR(\A_0, \b)$ and their antipodal points $-\cR(\A_0, \b)$ (see~\cref{fig:geometry-G-IRC}).
	
	\subsection{Comparison of BP and OMP for instance recovery}
	\label{sec:comparison_BP_OMP}
	
	Although BP and OMP are different strategies for solving the sparse recovery problem, our analysis provides a means to compare their ability for finding subspace-preserving representations.  
	Specifically, both the T-IDC for BP (see~\cref{eq:T-IDC}) and the IRC for OMP (see \cref{eq:IRC}) require the condition $\|\A_-^\transpose \v\|_\infty < \|\A_0^\transpose \v\|_\infty$ to be satisfied for certain $\v$\footnote{\label{ftnt:eqivalency}We use the fact that $\|\A_0^\transpose \v\|_\infty = 1$ for any $\v \in \BP_{\Dual}(\A_0, \b)$. To see why this holds, note that $P_{\cS_0}(\v) \in \cD(\A_0, \b)$ where $P_{\cS_0}(\v)$ is the projection of $\v$ onto $\cS_0$, and that $\cD(\A_0, \b)$ is a face of $\cK_0^o$, i.e., $\cD(\A_0, \b)\subseteq \{\v\in \cS_0: \|\A_0^\transpose \v\|_\infty = 1\}$. Therefore, it has $\|\A_0^\transpose \v\|_\infty = \|\A_0^\transpose P_{\cS_0}(\v)\|_\infty = 1$.}.
	In particular, BP requires the condition to be satisfied for all $\v \in \BP_{\Dual}(\A_0, \b)$, while OMP requires the condition to be satisfied for all $\v \in \cR(\A_0, \b)$. 
	
	The G-IDC for BP (see \cref{eq:G-IDC} and \cref{def:angular}) and the G-IRC for OMP (see~\cref{eq:G-IRC} and \cref{def:angular.omp}) allow for a geometrically interpretable comparison of instance recovery.
	Specifically, the only difference between \cref{def:angular} and \cref{def:angular.omp} is that BP requires the condition to be satisfied by at least one $\v \in \cD(\A_0, \b)$  while OMP requires the condition to be satisfied by all $\v \in \cR(\A_0, \b)$. 
	Geometrically, the condition for BP gives a pair of antipodal spherical caps centered at $\pm\v$ for $\v \in \cD(\A_0, \b)$ (assuming $\cD(\A_0, \b)$ contains a single point), while the condition for OMP gives a collection of $d_0$ pairs of antipodal spherical caps centered at $\pm\v$ for each $\v \in \cR(\A_0, \b)$ (assuming $\cR(\A_0, \b)$ contains $d_0$ points). 
	Then, instance recovery by BP and OMP are guaranteed if no point from $\cA_-$ lies in any of their respective spherical caps. 
	If we assume that points in $\cA_-$ are distributed independently of points in $\cA_0$ and the vector $\b$, then the G-IDC is satisfied with higher probability than the G-IRC.

	\section{Universal Recovery Conditions}
	\label{sec:universal_recovery}
	In many applications one is concerned whether subspace-preserving recovery is possible not only for one point in a subspace but for all points in a subspace, i.e., to achieve universal recovery. 
	In sparse representation-based face recognition, for example, a universal recovery guarantee means that the classifier is able to correctly classify all possible face test images to the classes that they belong to.
	Unlike the instance recovery conditions developed in \cref{sec:instance_recovery} that depend on both $\A$ and $\b$, the recovery conditions in this section ensure subspace-preserving representations for all $\b \in \cS_0\setminus \{\0\}$, and therefore rely only on $\A$. 
	An overview of the recovery conditions developed in this section is found in \cref{fig:result-flowchart}.

	\subsection{Universal recovery conditions for BP}
	\label{sec:universal_recovery_condition_BP}
	
	One way to derive universal recovery conditions is to require the instance recovery conditions derived in \cref{sec:instance_recovery} to hold for all $\b \in \cS_0 \setminus \{\0\}$. 
	In \cref{sec.universal.tight.bp}, we adopt such an approach to derive a tight condition for universal recovery by BP from the T-IDC in \cref{thm:T-IDC}.
	We then give a sufficient condition for recovery in \cref{sec.universal.suff.bp} and two weaker sufficient conditions for recovery in \cref{sec.universal.geom.bp} that have clear geometric interpretations.
	
	\subsubsection{A tight condition}\label{sec.universal.tight.bp}
	
	Recall that the T-IDC in \cref{thm:T-IDC} is necessary and sufficient for instance recovery by BP. 
	We may require that the T-IDC is satisfied for all $\b \in \cS_0\setminus \{\0\}$, which gives a tight condition for universal recovery by BP. 
	\begin{lemma}\label{thm:universal-BP-naive}
		All elements in $\BP(\A, \b)$ are subspace-preserving representations for all $\b \in \cS_0\setminus \{\0\}$ if and only if 	\begin{equation}\label{eq:universal-BP-naive}
		\forall \b \in \cS_0\setminus \{\0\}, \ \exists \v \in \BP_{\Dual}(\A_0, \b)
		\ \ \text{satisfying} \ \
		\|\A_-^\transpose \v\|_\infty < 1.
		\end{equation} 
	\end{lemma}
	
	Condition~\cref{eq:universal-BP-naive} does not provide insight since it is not directly related to the properties of $\cA$. 
	In the following, we derive a condition equivalent to \cref{eq:universal-BP-naive} that provides insight. 
	We need the following notion of \emph{dual points} associated with $\cA_0$.

	\begin{definition}[dual points associated with $\A_0$]\label{def:dual-point}
		The set of dual points associated with $\A_0$, denoted as $\cD(\A_0)$, is defined as the extreme points of $\cK^o_0$.
	\end{definition}
	\cref{def:dual-point} defines dual points $\cD(\A_0)$ associated with $\A_0$, which needs to be distinguished from  \cref{def:instance-dual-point} that defines dual points $\cD(\A_0, \b)$ associated with $\A_0$ and $\b$. 
	Geometrically, the set of dual points $\cD(\A_0, \b)$ corresponds to a face  of the polyhedron $\cK^o_0$ that is determined by $\b$, while the set of dual points $\cD(\A_0)$ is the set of all vertices (i.e., all faces of dimension $0$) of the polyhedron $\cK^o_0$ (see \cref{fig:geometry-2D}).
	
	The following result provides an upper bound on the size of $\cD(\A_0)$.
	\begin{lemma}\label{thm:dual-finite}
		The set $\cD(\A_0)$ is finite. Specifically, 
		\begin{equation}
		\card (\cD(\A_0)) \le 2^{d_0}  \binom{N_0}{d_0}
		\end{equation}
		with $N_0 := \card(\cA_0)$ denoting the number of data points in $\cA_0$.
	\end{lemma}
	\begin{proof}
		Consider a linear program with variable $\v$, constraint $\v \in \cK^o_0$, and an arbitrary linear objective function. Since the set of dual points $\cD(\A_0)$ is by definition the set of extreme points of $\cK^o_0$, they are the same as the basic feasible solutions of the linear program \cite{Nocedal:06}. Since each basic feasible solution is determined by $d_0$ linearly independent constraints from among the $2 N_0$ constraints of $\|\A_0^{\transpose} \v\|_\infty \le 1$, there are at most $2^{d_0} \binom{N_0}{d_0}$ ways to choose such a set of constraints (we use that at most one of the two constraints $-1 \leq (\A_0^{\transpose} \v)_i \le 1$ can be chosen for each $i\in\{1,2,\dots,N_0\}$).
	\end{proof}
	The bound in \cref{thm:dual-finite} is not tight in general since not every set of constraints in the $2^{d_0} \binom{N_0}{d_0}$ combinations produces a basic feasible solution. For example, in \cref{fig:geometry-2D} the combination of constraints $\langle \a_1, \v \rangle \le 1$ and $\langle \a_2, \v \rangle \le 1$ produces a basic feasible solution, while the combination of constraints $\langle \a_1, \v \rangle \le 1$ and $\langle \a_3, \v \rangle \le 1$ does not. 
	Nonetheless, this bound is sufficient for the purpose of this paper\footnote{There are tighter but more sophisticated bounds in the literature, see e.g., \cite{Fukuda:04,Novik:19}.}.
	
	We now state a result with a tight condition for universal recovery by BP. 
	
	\begin{theorem}[Tight Universal Dual Condition (T-UDC)]\label{thm:T-UDC}
		All elements in \\$\BP(\A, \b)$ are subspace-preserving for all $\b \in \cS_0\setminus \{\0\}$ if and only if 		
		\begin{equation}\label{eq:T-UDC}
		\forall \w \in \cD(\A_0), \ \exists \v \in (\w + \cS_0^\perp)
		\ \ \text{satisfying} \ \
		\|\A_-^\transpose \v\|_\infty < 1.
		\end{equation} 
	\end{theorem}
	\begin{proof}
		In view of \cref{thm:universal-BP-naive}, we only need to show that \cref{eq:universal-BP-naive} and \cref{eq:T-UDC} are equivalent. 
		Note from \cref{eq:dual-solution-dual-point} that \cref{eq:universal-BP-naive} is equivalent to 
		\begin{equation}\label{eq:universal-BP-naive-equivalent}
		\forall \b \in \cS_0\setminus \{\0\}, \ \exists \w \in \cD(\A_0, \b)	
		\ \text{and} \ \v \in (\w + \cS_0^\perp)		
		\ \ \text{satisfying} \ \ 
		\|\A_-^\transpose \v\|_\infty < 1.
		\end{equation}
		Therefore, we only need to show that \cref{eq:T-UDC} and \cref{eq:universal-BP-naive-equivalent} are equivalent.
		
		To show that \cref{eq:universal-BP-naive-equivalent} implies \cref{eq:T-UDC}, take any $\w \in \cD(\A_0)$. 
		We will find a $\v \in (\w + \cS_0^\perp)$ such that $\|\A_-^\transpose \v\|_\infty < 1$. 
		For that purpose, let us define $\b = \w \in \cS_0\setminus \{\0\}$. 
		We have $\cD(\A_0, \b) = \Argmax_{\v \in \cK_0^o}\langle \b, \v\rangle = \{\b\}$, where the first equality follows from \cref{def:instance-dual-point} and the second equality follows from the fact that $\b = \w\in\cD(\A_0)$, i.e., $\b$ is an extreme point of $\cK_0^o$. 
		Therefore, $\w$ is the only element in $\cD(\A_0, \b)$. By \cref{eq:universal-BP-naive-equivalent}, there exists a $\v \in (\w + \cS_0^\perp)$ such that $\|\A_-^\transpose \v\|_\infty < 1$. This shows that \cref{eq:T-UDC} holds.
		
		We now prove that \cref{eq:T-UDC} implies \cref{eq:universal-BP-naive-equivalent}. Take any $\b \in \cS_0\setminus \{\0\}$. Note that $\cD(\A_0, \b)$ is a face of $\cK_0^o$, and therefore it must contain at least a vertex of $\cK_0^o$. That is, there exists a $\w \in \cD(\A_0)$ such that $\w \in \cD(\A_0, \b)$. From \cref{eq:T-UDC}, there exists a $\v \in (\w + \cS_0^\perp)$ such that $\|\A_-^\transpose \v\|_\infty < 1$. This shows that \cref{eq:universal-BP-naive-equivalent} holds. 
	\end{proof}	
	
	\cref{thm:T-UDC} states that BP produces subspace-preserving solutions for all $\b \in \cS_0\setminus \{\0\}$ as long as for each dual point $\w \in \cD(\A_0)$ there exists a point on the affine subspace $\w + \cS_0^\perp$ that has inner product (in absolute value) less than one for all points in $\cA_-$.
	In order to establish connections between \cref{thm:T-UDC} and existing results, we present two equivalent conditions to the T-UDC in \cref{eq:T-UDC} that are characterized by the objective value of the BP optimization problem and the null space of $\A$. 
		\begin{theorem}[Equivalent forms of T-UDC]\label{thm:T-UDC-equivalent}
			The T-UDC holds if and only if either of the following two conditions holds:
			\begin{enumerate}[label=(\roman*)]
				\item $\forall \b \in \cS_0 \setminus \{\0\}$, it holds that 
				\begin{equation}\label{eq:equivalent-TUDC-obj}
				\min_{\c_0: \A_0 \c_0 = \b}\|\c_0\|_1 < \min_{\c_-: \A_- \c_- = \b}\|\c_-\|_1.
				\end{equation}
				\item $\forall \v = (\v_0, \v_-)$ in the null space of $\A$ satisfying $\v_- \ne \0$, it holds that
				\begin{equation}\label{eq:equivalent-TUDC-null}
				\min_{\c_0: \A_0 \c_0 = \A_0 \v_0}\|\c_0\|_1 < \|\v_-\|_1.
				\end{equation}
			\end{enumerate}
		\end{theorem}
		The condition in \cref{eq:equivalent-TUDC-obj} states that the objective value of BP with signal $\b$ and dictionary $\A_0$ is strictly smaller than that with signal $\b$ and dictionary $\A_-$.
		In the context of subspace clustering, it is known that \cref{eq:equivalent-TUDC-obj} is an equivalent condition for universal recovery condition by BP \cite{Elhamifar:TPAMI13}. 
		Combining this with \cref{thm:T-UDC} we know that \cref{eq:equivalent-TUDC-obj} is equivalent to the T-UDC in \cref{eq:T-UDC}. 
		On the other hand, the condition in \cref{eq:equivalent-TUDC-null}, which is characterized by properties associated with the null space of $\A$, has not appeared before to the best of our knowledge. 
		As we will see in \cref{sec:sparse-recovery}, \cref{eq:equivalent-TUDC-null} is closely related to the null space property in the study of sparse signal recovery.
		\begin{proof}[Proof of \cref{thm:T-UDC-equivalent}]
			The previous paragraph shows that \cref{eq:equivalent-TUDC-obj} is equivalent to the T-UDC. 
			To establish the theorem, we only need to show that the conditions \cref{eq:equivalent-TUDC-obj} and \cref{eq:equivalent-TUDC-null} are equivalent.
			
			We first show that \cref{eq:equivalent-TUDC-obj} implies \cref{eq:equivalent-TUDC-null}.
			To that end, let $\v = (\v_0, \v_-)$ be any vector in the null space of $\A$ that satisfies $\v_- \ne \0$, so that in particular it holds that $\A_0 \v_0 + \A_- \v_- = 0$. 
			If $\A_0 \v_0 = \0$, then the left-hand side of \cref{eq:equivalent-TUDC-null} becomes zero. 
			Combining this with $\|\v_-\|_1 > 0$ shows that \cref{eq:equivalent-TUDC-null} holds. 
			Otherwise, we may let $\b = \A_0 \v_0 \ne \0$ and apply \cref{eq:equivalent-TUDC-obj}, which gives
			\begin{align*}
				\min_{\c_0: \A_0 \c_0 = \A_0 \v_0}\|\c_0\|_1 
				&< \min_{\c_-: \A_- \c_- = \A_0 \v_0}\|\c_-\|_1 \\ 
				&= \min_{\c_-: \A_- \c_- = \A_- (-\v_-)}\|\c_-\|_1 
				\le \|-\v_-\|_1 = \|\v_-\|_1,
			\end{align*}
			which establishes that~\eqref{eq:equivalent-TUDC-null} holds.
			
			To show that \cref{eq:equivalent-TUDC-null} implies \cref{eq:equivalent-TUDC-obj}, let $\b\in \cS_0\setminus \{\0\}$. 
			Note that the optimization problem on the left-hand side of \cref{eq:equivalent-TUDC-obj} is feasible.
			Therefore, \cref{eq:equivalent-TUDC-obj} trivially holds if the optimization problem on its right-hand side is not feasible. 
			Otherwise, we can define		
			\begin{equation}\label{eq:prf-tudc-equivalent}
			\v_- \in -\Argmin_{\c_-: \A_- \c_- = \b}\|\c_-\|_1
			\end{equation}
			and let $\v_0$ be any vector satisfying $\A_0 \v_0 = \b$ (this system is feasible as stated above).  We can now observe that $\A_0 \v_0 + \A_- \v_- = \0$ (i.e., $(\v_0, \v_-)$ is in the null space of $\A$) and $\v_- \neq \0$ since $\b \neq \0$, which means that we may apply \cref{eq:equivalent-TUDC-null} to obtain
			\begin{equation*}
			\min_{\c_0: \A_0 \c_0 = \b_0}\|\c_0\|_1 
			= \min_{\c_0: \A_0 \c_0 = \A_0 \v_0}\|\c_0\|_1 
			< \|\v_-\|_1 
			= \|-\v_-\|_1 
			= \min_{\c_-: \A_- \c_- = \b}\|\c_-\|_1,
			\end{equation*}
			where the last equality follows from \cref{eq:prf-tudc-equivalent}. This completes the proof.
		\end{proof}
	
	\subsubsection{A sufficient condition}\label{sec.universal.suff.bp}
	
	Note that condition \cref{eq:T-UDC} is equivalent to
	\begin{equation}\label{eq:T-UDC-reformulate}
	\forall \w \in \cD(\A_0), \ \exists \v \in (\w + \cS_0^\perp)
	\ \ \text{satisfying} \ \
	\max_{\a \in \cA_-} \left|\left\langle \frac{\v}{\|\v\|_2}, \a \right\rangle\right| < \frac{1}{\|\v\|_2}.
	\end{equation} 
	As with our discussion for instance recovery by BP, 
	the left-hand side of the inequality in \cref{eq:T-UDC-reformulate} captures the similarity between $\v$ and $\cA_-$, but the right-hand side can be arbitrarily small. 
	Therefore, we restrict our attention to those $\v$ with smallest $\ell_2$-norm. 
	It is easy to see that such $\v$ is given by $\v = \w$ for each $\w \in \cD(\A_0)$.
	This leads to the following result that gives a sufficient condition for universal recovery by BP.
	
	\begin{theorem}[Universal Dual Condition (UDC)]\label{thm:UDC}
		All elements in $\BP(\A, \b)$ are subspace-preserving for all $\b \in \cS_0\setminus \{\0\}$ if
		\begin{equation}\label{eq:UDC}
		\forall \v \in \cD(\A_0)
		\ \text{it holds that} \
		\|\A_-^\transpose \v\|_\infty < 1.
		\end{equation}
	\end{theorem}
	\begin{proof}
		We prove that condition \cref{eq:UDC} implies condition \cref{eq:T-UDC}, which is sufficient to establish the desired result because of \cref{thm:T-UDC}.  To that end, assume that~\cref{eq:UDC} holds and let $\w \in \cD(\A_0)$. Then, choose $\v = \w$ so that $\v \in \w + \cS_0^\perp$ trivially holds, and from \cref{eq:UDC} we know that $\|\A_-^\transpose \v\|_\infty < 1$. Thus, condition~\cref{eq:T-UDC} holds. 
	\end{proof}
	
	Using \cref{thm:UDC}, we now derive two geometrically interpretable conditions.
	
	\subsubsection{Two geometrically intuitive sufficient conditions}\label{sec.universal.geom.bp}
	
	For any $\v \in \cD(\A_0) \subseteq\cK_0^o$, it follows from \cref{def:circumradius} and \cref{thm:inradius-coveringradius-circumradius} that
	\begin{equation}\label{eq:prf-bound-dual-A0}
	\|\v\|_2 \le R_0 = 1 / r_0.
	\end{equation}
	Combining \cref{eq:prf-bound-dual-A0} with \cref{eq:UDC} we have the following theorem.
	
	\begin{theorem}[Geometric Universal Dual Condition (G-UDC)\footnote{The G-UDC (resp., the G-USC) is equivalent to the dual recovery condition (resp., the principal recovery condition) in \cite{You:ICML15}. We changed these names so as to have a consistent naming convention.\label{fn:rename}}]\label{thm:G-UDC}
		All elements in $\BP(\A, \b)$ are subspace-preserving for any $\b \in \cS_0\setminus \{\0\}$ if
		\begin{equation}\label{eq:G-UDC}
		\forall \v \in \cD(\A_0) 
		\ \text{it holds that} \
		r_0 > \mu(\v, \cA_-),
		\end{equation}
		with \cref{eq:G-UDC} holding if and only if
		\begin{equation}\label{eq:G-UDC-angular}
		\forall \v \in \cD(\A_0) 
		\ \text{it holds that} \
		\gamma_0 < \theta(\{\pm\v\}, \cA_-).
		\end{equation}
	\end{theorem}
	\begin{proof}
		We prove that condition \cref{eq:G-UDC} holding implies that condition \cref{eq:UDC} holds, which is sufficient due to \cref{thm:UDC}.  To this end, assume that \cref{eq:G-UDC} holds and let $\v \in \cD(\A_0)$. It then follows from \cref{eq:prf-bound-dual-A0}, \cref{eq:G-UDC} and \cref{eq:def-coherence} that 
		\begin{equation}
		1/\|\v\|_2 \ge r_0 > \mu(\v, \cA_-) = \|\A_-^\transpose \v/\|\v\|_2\|_\infty,
		\end{equation}
		which means that condition \cref{eq:UDC} is satisfied, as claimed. 
		The fact that conditions \cref{eq:G-UDC} and \cref{eq:G-UDC-angular} are equivalent follows from \cref{thm:inradius-coveringradius-circumradius}.
	\end{proof}
	
	Condition \cref{eq:G-UDC} requires the points in $\cA_0$ to be sufficiently well distributed as measured by the inradius of $\pm\cA_0$, and the dual points $\cD(\A_0)$ to be sufficiently separated from points in $\cA_-$ as measured by their coherence.
	The equivalent condition given by \cref{eq:G-UDC-angular} states that all points from $\cA_-$ do not lie in the union of spherical caps of radius $\gamma_0$ centered at each of the points from $\cD(\A_0)$ (see~\cref{fig:geometry-3D-DRC}). 
	By noting that $\cD(\A_0)$ is a subset of $\cS_0$, we immediately have the following result.
	
	\begin{theorem}[Geometric Universal Subspace Condition (G-USC)\footref{fn:rename}]\label{thm:G-USC}
		All \\elements in $\BP(\A, \b)$ are subspace-preserving for all $\b \in \cS_0\setminus \{\0\}$ if
		\begin{equation}\label{eq:G-USC}
		\forall \v \in \cS_0 
		\ \text{it holds that} \
		r_0 > \mu(\v, \cA_-),
		\end{equation}
		with \cref{eq:G-USC} holding if and only if
		\begin{equation}\label{eq:G-USC-angular}
		\forall \v \in \cS_0
		\ \text{it holds that} \
		\gamma_0 < \theta(\{\pm\v\}, \cA_-).
		\end{equation}
	\end{theorem}
	Geometrically, condition \cref{eq:G-USC-angular} states that all points from $\cA_-$ do not lie within an angular distance of $\gamma_0$ to the subspace $\cS_0$ (see~\cref{fig:geometry-3D-PRC}). 
	

	\subsection{Universal recovery conditions for OMP}
	\label{sec:universal_recovery_condition_OMP}
	
	Similar to the case of BP, we may derive universal recovery conditions for OMP by requiring that the instance recovery conditions for OMP presented in \cref{sec:instance_recovery_OMP} are satisfied for all $\b \in \cS_0 \setminus \{\0\}$.  
	Following this idea, we present a universal residual condition (URC) in \cref{sec.universal.suff.omp}, which is a sufficient condition for universal recovery by OMP. 
	We further establish, perhaps surprisingly, that the URC is equivalent to the UDC in \cref{eq:UDC}, showing that the same condition guarantees universal recovery of both BP and OMP. 
	By utilizing such an equivalency, we further show in \cref{sec.universal.geom.omp} that the G-UDC \cref{eq:G-UDC} and G-USC \cref{eq:G-USC} also guarantee universal recovery by OMP.
	
	\subsubsection{A sufficient condition}\label{sec.universal.suff.omp}
	
	By requiring that the IRC in \cref{thm:IRC} be satisfied for all $\b \in \cS_0 \setminus \{\0\}$, we get a sufficient condition for universal recovery by OMP as stated in the following lemma.
	\begin{lemma}\label{thm:universal-OMP-naive}
		The solution $\OMP(\A, \b)$ is subspace-preserving for all $\b \!\in\! \cS_0 \!\setminus \!\{\0\}$ if
		\begin{equation}\label{eq:universal-OMP-naive}
		\forall \v \in \cup_{\b \in \cS_0 \setminus \{\0\}} \cR(\A_0, \b)
		\ \text{it holds that} \
		\|\A_-^\transpose \v\|_\infty < \|\A_0^\transpose \v\|_\infty.
		\end{equation} 
	\end{lemma}
	
	The following result shows that the set $\cup_{\b \in \cS_0 \setminus \{\0\}} \cR(\A_0, \b)$ appearing in \cref{eq:universal-OMP-naive} is simply a disguised mathematical way of describing the subspace $\cS_0$.
	
	\begin{lemma}\label{lem:disguised}
		It holds that $\cup_{\b \in \cS_0 \setminus \{\0\}} \cR(\A_0, \b) = \cS_0 \setminus \{\0\}$.
	\end{lemma}
	\begin{proof}
		We can use the fact that $\cR(\A_0, \b)$ is a subset of $\cS_0\setminus \{\0\}$ for all $\b \in \cS_0 \setminus \{\0\}$ (see the discussion in \cref{sec.suff.omp}) to conclude that
		\begin{equation}\label{eq:residual-union-left}
		\cup_{\b \in \cS_0 \setminus \{\0\}} \cR(\A_0, \b) \subseteq \cS_0  \setminus \{\0\}.
		\end{equation} 
		Note from \cref{alg.omp} that the residual vector $\v^{(0)}$ computed during the first iteration of OMP with data $\A_0$ and $\b$ is equal to $\b$, i.e., $\v^{(0)} = \b$. 
		Therefore, by \cref{def:residual-points} we know that $\b \in \cR(\A_0, \b)$ for all $\b \in \cS_0 \setminus \{\0\}$. 
		It follows, by considering all possible $\b \in \cS_0 \setminus \{\0\}$, that
		\begin{equation}\label{eq:residual-union-right}
		\cS_0 \setminus \{\0\} \subseteq \cup_{\b \in \cS_0 \setminus \{\0\}} \cR(\A_0, \b).
		\end{equation} 
		The desired result follows by combining \cref{eq:residual-union-left} and \cref{eq:residual-union-right}.
	\end{proof}
	
	\cref{lem:disguised} is now used to give the following universal recovery result for OMP. 
	\begin{theorem}[Universal Residual Condition (URC)]\label{thm:URC}
		The solution $\OMP(\A, \b)$ is subspace-preserving for all $\b \in \cS_0 \setminus \{\0\}$ if
		\begin{equation}\label{eq:URC}
		\forall \v \in \cS_0\setminus \{\0\}
		\ \text{it holds that} \
		\|\A_-^\transpose \v\|_\infty < \|\A_0^\transpose \v\|_\infty.
		\end{equation}
	\end{theorem}
	\begin{proof}
		The theorem follows by combining \cref{thm:universal-OMP-naive} with \cref{lem:disguised}.
	\end{proof}
	
	The URC \cref{eq:URC} for universal recovery by OMP may be compared with the UDC \cref{eq:UDC} for universal recovery by BP. 
	They both require that $\|\A_-^\transpose \v\|_\infty < \|\A_0^\transpose \v\|_\infty$ is satisfied for certain $\v$\footnote{We use the fact that $\|\A_0^\transpose \v\|_\infty = 1$ for any $\v \in \cD(\A_0)$.}. 
	However, the condition for BP must be satisfied for all $\v \in \cD(\A_0)$, while the condition for OMP must be satisfied for all $\v \in \cS_0$. 
	Since $\cD(\A_0)$ is a subset of $\cS_0$, it is clear that the UDC is implied by the URC. It is surprising, however, that the conditions are actually equivalent, as we now set out to prove.
	
	

	We need to use the result that the polar set $\cK_0^o$ is a symmetric convex body that induces a norm on the space $\cS_0$, by means of the Minkowski functional. The relevant definition and accompanying result are stated next.
	
	\begin{definition}[Minkowski functional]
		The Minkowski functional of a given set $\cK$ is defined over $\spann(\cK)$ as
		$\|\v\|_\cK := \inf\{t>0: \tfrac{1}{t}\v \in \cK\}$.
	\end{definition}
	\begin{lemma}\label{lem:unitball} 
		The function $\| \cdot \|_{\cK_0^o}$ is a norm defined over $\cS_0$ with unit ball $\cK_0^o$.
	\end{lemma}
	\begin{proof}
		The desired result follows from~\cite{Vershynin:09} using the fact that $\cK_0^o$ is a symmetric convex body and $\spann(\cK_0^o) = \spann(\cK_0) = \spann(\cA_0) =  \cS_0$.
	\end{proof}
	
	
	We may now establish that the UDC and URC are equivalent.

	\begin{theorem}\label{thm:UCC-equivalent-condition}
		The UDC \cref{eq:UDC} and the URC \cref{eq:URC} are equivalent. 
	\end{theorem} 
	\begin{proof}
		To show that the UDC implies the URC, we need to take any $\v \in \cS_0 \setminus \{\0\}$ and show that $\|\A_-^\transpose \v\|_\infty < \|\A_0^\transpose \v\|_\infty$.
		Equivalently, by defining $\bar{\v} := \v/\|\v\|_{\cK_0^o}$, we may show that $\|\A_-^\transpose \bar{\v}\|_\infty < \|\A_0^\transpose \bar{\v}\|_\infty$, which will be our goal.
		
		\cref{def:dual-point} states that $\cD(\A_0)$ is the set of extreme points of $\cK_0^o$. 
		Since the convex hull of the set of the extreme points of a convex
		body is this convex body itself~\cite{Brazitikos:14}, we know that $\cK_0^o$ is the convex hull of $\cD(\A_0)$. 
		Thus, since $\bar{\v} \in \cK_0^o$ (see \cref{lem:unitball}), we may express $\bar{\v}$ as a finite convex combination of points from $\cD(\A_0)$, i.e., 
		$\bar{\v} = \sum_i \alpha_i  \v_i $ with $\v_i \in \cD(\A_0)$ and $\alpha_i \in [0, 1]$ for all $i$ and $\sum_i \alpha_i= 1$. This gives
		\begin{equation}
		\|\A_-^\transpose \bar{\v}\|_\infty 
		= \|\A_-^\transpose \sum_i \alpha_i\v_i \|_\infty
		\le \sum_i \alpha_i\|\A_-^\transpose \v_i   \|_\infty
		< \sum_i \alpha_i 
		= 1
		= \|\A_0^\transpose \bar{\v}\|_\infty,
		\label{eq:proof-OMP-dual}
		\end{equation}
		where the strict inequality follows from the UDC and the final equality follows from the fact that $\|\bar{\v}\|_{\cK_0^o} = 1$. 
		This shows that the URC holds, as intended.
		
		To show that the URC implies the UDC, let $\v \in \cD(\A_0) \subseteq \cS_0 \setminus \{\0\}$. 
		From the URC we get $\|\A_-^\transpose \v\|_\infty < \|\A_0^\transpose \v\|_\infty$. 
		Also, by noting that  $\v \in \cD(\A_0) \subseteq \cK_0^o$ we get $\|\A_0^\transpose \v\|_\infty \le 1$. 
		Combining these two results gives $\|\A_-^\transpose \v\|_\infty < 1$, which is the UDC.
	\end{proof}

	\subsubsection{Two geometrically interpretable conditions}\label{sec.universal.geom.omp}  
	
	We can deduce that our geometrically interpretable conditions for universal subspace recovery for BP are also applicable to OMP by making the following  observations: 
	(i) the G-USC~\cref{eq:G-USC} and G-UDC \cref{eq:G-UDC} are sufficient conditions for ensuring the UDC \cref{eq:UDC} holds (see \cref{fig:result-flowchart}, \cref{thm:G-UDC,thm:G-USC});
	(ii) the UDC is equivalent to the URC (see \cref{fig:result-flowchart,thm:UCC-equivalent-condition}); and
	(iii) the URC is a sufficient condition for universal recovery for OMP (see~\cref{fig:result-flowchart,thm:URC}).
	We formally state this observation next.
	
	%
	
	\begin{corollary}\label{thm:G-UDC-OMP}
		The solution $\OMP(\A, \b)$ is subspace-preserving for all $\b \in \cS_0\setminus \{\0\}$ if either the G-UDC in \cref{eq:G-UDC} holds or the G-USC in \cref{eq:G-USC} holds.
	\end{corollary}
	


	\section{Sparse Signal Recovery}
	\label{sec:sparse-recovery}
	In this section, we show that the instance and universal recovery conditions for subspace-preserving recovery can also provide correctness guarantees for sparse signal recovery. To formally state such results, we identify three fundamental problems that are studied in compressive sensing. 
	
	\begin{problem}[instance non-uniform recovery problem]
		What conditions on $\cA$ ensure that a \emph{particular} $d_0$-sparse vector $\c$ that has nonzero entries corresponding to $\cA_0 \subseteq \cA$ can be recovered from the linear measurements $\b = \A \c$?
	\end{problem}
	\begin{problem}[universal non-uniform recovery problem]
		What conditions on $\cA$ ensure  that \emph{any} $d_0$-sparse vector $\c$ that has nonzero entries corresponding to $\cA_0 \subseteq \cA$ can be recovered from the linear measurements $\b = \A \c$?
	\end{problem}
	\begin{problem}[uniform recovery problem]
		What conditions on $\cA$ ensure  that \emph{any} $d_0$-sparse vector $\c$ can be recovered from the linear measurements $\b = \A \c$?
	\end{problem}
	
	In the literature, the first two problems are sometimes referred to as the non-uniform recovery problem and the third problem as the uniform recovery problem \cite{Foucart:13}. 
	To further distinguish the first two problems, we will refer to them as the instance and universal non-uniform recovery problems, respectively. 
	The dual certificate condition \cite{Fuchs:TIT04,Zhang:ACM16}, exact recovery condition \cite{Tropp:TIT04} and mutual coherence condition \cite{DonohoElad:PNAS03,Tropp:TIT04} are, respectively, well-known conditions that answer these three questions.
	
	Observe that if $\cA_0$ contains linearly independent data points, then the instance and universal \emph{subspace-preserving recovery} problems are equivalent to the instance and universal non-uniform \emph{sparse signal recovery} problems, respectively. 
	More precisely, BP (resp., OMP) recovers a $d_0$-sparse vector $\c$ that has nonzero entries corresponding to $\cA_0$ from $\b = \A \c$ if and only if all solutions in $\BP(\A_0, \b)$ (resp., $\OMP(\A_0, \b)$) are subspace-preserving. 
	Likewise, BP (resp., OMP) recovers any $d_0$-sparse vector $\c$ with nonzero entries corresponding to $\cA_0$ from its measurement $\b = \A \c$ if and only if all solutions to $\BP(\A_0, \b)$ (resp., $\OMP(\A_0, \b)$) are subspace-preserving for any $\b \in \cS_0$. 
	
	As a consequence of this equivalency, our instance and universal recovery conditions also provide guarantees for sparse signal recovery. 
	In this section, we establish that some of our conditions are equivalent to existing conditions in the sparse signal recovery literature.
	In such cases, our conditions are generalizations of the corresponding conditions for sparse signal recovery. 
	In the other cases, our conditions provide new sparse signal recovery conditions that bring new insight to the area.
	
	\subsection{Conditions for instance non-uniform recovery}
	
	By applying \cref{thm:T-IDC,thm:IDC,thm:IRC}, we obtain instance non-uniform sparse signal recovery results for BP and OMP. Similarly, we obtain geometrically interpretable conditions by using \cref{thm:G-IDC,thm:G-IRC}. These results are now summarized.
	
	\begin{theorem} \label{thm:inur}
		Let the columns of $\A_0$ be linearly independent,  $\c$ be a vector supported on $\A_0$ and $\b = \A \c$. Then, the following hold:
		\begin{enumerate}[label=(\roman*)]
			\item \label{itm:inur-T-IDC}$\c$ is the unique solution in $\BP(\A, \b)$ if and only if the T-IDC in \cref{eq:T-IDC} holds.
			\item \label{itm:inur-IDC}$\c$ is the unique solution in $\BP(\A, \b)$ if either the IDC in \cref{eq:IDC} holds or the G-IDC in \cref{eq:G-IDC} holds.
			\item \label{itm:inur-IRC}$\c$ is given by $\OMP(\A, \b)$ if either the IRC in \cref{eq:IRC} holds or the G-IRC in \cref{eq:G-IRC} holds.
		\end{enumerate}
	\end{theorem}
	
	%
	
	We note that when the T-IDC is considered under the assumption that the columns of $\A$ are linearly independent, it is equivalent to a well-known sparse signal recovery condition \cite[Theorem 1]{Fuchs:TIT04} called the \emph{dual certificate condition} in~\cite{Zhang:ACM16}.
	Therefore, although the result \cref{thm:inur} \ref{itm:inur-T-IDC} 
	does not provide a new result for sparse signal recovery, it does show that the T-IDC may be regarded as a generalization of the dual certificate condition to 
	the instance subspace-preserving recovery problem. 
	\cref{thm:inur} \ref{itm:inur-IDC} shows that the IDC is a sufficient condition for sparse signal recovery by BP that follows trivially by relaxing the condition in \cref{thm:inur} \ref{itm:inur-T-IDC}.
	This condition seldomly, if ever, appears in the sparse signal recovery literature.
	
	\cref{thm:inur} \ref{itm:inur-IRC} shows that the IRC is a sufficient condition for sparse signal recovery by OMP. 
	This condition has been implicitly used in the derivation of sparse recovery conditions for OMP (see, e.g. \cite{Tropp:TIT04}) but is not identified as an independent result. 
	
	%
	
	\cref{thm:inur} \ref{itm:inur-IDC}-\ref{itm:inur-IRC} also provide new geometrically interpretable conditions (i.e., the G-IDC and the G-IRC) for sparse signal recovery. They state that instance non-uniform recovery by BP and OMP is guaranteed to be successful if the points in $\cA_0$ are well distributed as measured by the inradius $r_0$ and the points in $\cA_-$ are sufficiently separated from (i) a certain point in $\cD(\A_0, \b)$ for BP, or (ii) all points in $\cR(\A_0, \b)$ for OMP. 
	To the best of our knowledge, such conditions are new to the sparse signal recovery literature. 
	
	Unlike the general subspace-preserving recovery problem for which the inradius $r_0$ is difficult to compute in general, there exists an easy way of computing $r_0$ for sparse signal recovery when the columns of $\A_0$ are linearly independent. 
	Since $r_0$ can be computed, one may check whether the G-IDC and G-IRC are satisfied for a given $\A$ and $\c$. 
	To show how $r_0$ may be computed, we first establish the following result.
	
	\begin{lemma}		\label{thm:sparse-recovery-dualpoint}
		Let the columns of $\A_0$ be linearly independent. Then, the set of dual points contains exactly $2^{d_0}$ points and is given by $\cD(\A_0) = \{\A_0 (\A_0^{\transpose} \A_0) ^{-1} \u : \u \in U_{d_0}\}$, where $U_{d_0} := \{ (u_1, \dots, u_{d_0}) :  u_i = \pm 1 \ \text{for} \ i = 1, \dots, d_0 \}$.
	\end{lemma}
	\begin{proof}
		From \cref{thm:dual-finite}, there are at most $2^{d_0}$ dual points in the case that $\A_0$ has full column rank. So in order to prove the result, it is enough to show that the set $\{\A_0 (\A_0^{\transpose} \A_0) ^{-1} \u : \u \in U_{d_0}\}$ contains $2^{d_0}$ points, and each of them is a dual point. 
		
		Note that $U_{d_0}$ has $2^{d_0}$ points.  Also, if $\{\u_1, \u_2\} \subseteq U_{d_0}$ with $\u_1 \ne \u_2$, then $\A_0 (\A_0^{\transpose} \A_0) ^{-1} \u_1 \ne \A_0 (\A_0^{\transpose} \A_0) ^{-1} \u_2$ because $\A_0 (\A_0^{\transpose} \A_0) ^{-1}$ has full column rank as a consequence of $\A_0$ having full rank; thus, $\{\A_0 (\A_0^{\transpose} \A_0) ^{-1} \u : \u \in U_{d_0}\}$ has $2^{d_0}$ points. 
		
		
		Next, we show that $\v_0 := \A_0 (\A_0^{\transpose} \A_0) ^{-1} \u_0$ is a dual point for any $\u_0 \in U_{d_0}$. By definition, we need to show that $\v_0$ is an extreme point of the set $\cK_0^o = \{\v \in \cS_0: \|\A_0^{\transpose} \v\|_\infty \le 1\}$. First, note that $\v_0 \in \cK_0^o$ because $\v_0\in\range(\A_0)$ and $\|\A_0^{\transpose} \v_0\|_\infty = \|\u_0\|_\infty = 1$. Second, suppose that there are two points $\{\v_1, \v_2\} \subseteq \cK_0^o$ such that 
		$\v_0 = (1-\lambda) \v_1 + \lambda \v_2$
		for some $\lambda \in (0, 1)$. 
		Since $\spann(\A_0 (\A_0^{\transpose} \A_0) ^{-1}) = \cS_0$ and $\{\v_1, \v_2\} \subseteq \cK_0^o \subseteq \cS_0$, there exists $\c_1$ and $\c_2$ satisfying $\v_k = \A_0 (\A_0^{\transpose} \A_0) ^{-1} \c_k$ for $k \in\{1, 2\}$. Combining this with the fact that $\A_0$ has full column rank reveals that
		\begin{equation}
		\u_0
		= \A_0^{\transpose}\v_0
		= (1-\lambda) \A_0^{\transpose}\v_1 + \lambda \A_0^{\transpose}\v_2
		= (1-\lambda) \c_1 + \lambda \c_2.
		\label{eq:proof-convex-cmb-coeff}
		\end{equation}
		Consider the $i$th entry of $\u_0$, namely $[\u_0]_i = (1-\lambda) [\c_1]_i + \lambda [\c_2]_i$. Since $[\u_0]_i = \pm 1$ and the right-hand side of \cref{eq:proof-convex-cmb-coeff} is a convex combination with $\lambda \in(0,1)$ of two points in $[-1, 1]$ (this can be seen from $\|\c_k\|_\infty = \|\A_0^\transpose \v_k\|_\infty=1$ for $k \in \{1, 2\}$ because $\{\v_1,\v_2\} \subseteq \cK_0^o$), it follows that 
		$[\c_1]_i = [\c_2]_i$. Since the index $i$ was arbitrary, it follows that $\c_1 = \c_2$, which implies $\v_1 = \v_2$, thus showing that $\v_0$ is an extreme point.
	\end{proof}
	
	We may now use \cref{thm:inradius-coveringradius-circumradius,thm:sparse-recovery-dualpoint} to compute $r_0$ as
	\begin{equation}
	\begin{split}
	r_0 &= 1/R_0 
	= 1/\max\{ \|\v\|_2: \v \in \cK_0^o \}
	= 1/\max\{ \|\v\|_2: \v \in \cD(\A_0) \} \\
	&= 1/\max\{\|\A_0 (\A_0^{\transpose} \A_0) ^{-1} \u\|_2: \u \in U_{d_0}\},
	\end{split}
	\end{equation}
	which requires computing  
	$\|\A_0 (\A_0^{\transpose} \A_0) ^{-1} \u\|_2$ over all $\u \in U_{d_0}$. 
	However, no polynomial algorithm in $d_0$ for doing so exists since calculating $\max\{\|\A_0 (\A_0^{\transpose} \A_0) ^{-1}  \u\|_2: \u \in U_{d_0}\} = \|\A_0 (\A_0^{\transpose} \A_0) ^{-1} \|_{\infty, 2}$ is known to be NP-hard  \cite{Tropp:2004thesis}.
	
	\subsection{Conditions for universal non-uniform recovery}
	
	By applying \cref{thm:T-UDC,thm:UDC,thm:URC}, and noting from \cref{thm:UCC-equivalent-condition} that the UDC and the URC are equivalent, we obtain universal non-uniform sparse signal recovery results for BP and OMP. 
	Moreover, we can apply \cref{thm:G-UDC,thm:G-USC,thm:G-UDC-OMP} to obtain geometrically interpretable conditions. 
	These results are summarized as follows.
	
	\begin{theorem} \label{thm:iur}
		Let the columns of $\A_0$ be linearly independent. The following hold:$\!\!$
		\begin{enumerate}[label=(\roman*)]
			\item \label{itm:iur-T-UDC}Any $\c$ supported on $\A_0$ is the unique solution in $\BP(\A, \b)$ where $\b = \A \c$ if and only if the T-UDC in \cref{eq:T-UDC} holds.
			\item \label{itm:iur-UDC}Any $\c$ supported on $\A_0$ is the unique solution in $\BP(\A, \b)$ and $\OMP(\A, \b)$ where $\b = \A \c$ if either the UDC in \cref{eq:UDC} (equivalently, the URC in \cref{eq:URC}) holds, or the G-UDC in \cref{eq:G-UDC} holds, or the G-USC in \cref{eq:G-USC} holds.
		\end{enumerate}
	\end{theorem}
	
	\cref{thm:iur} \ref{itm:iur-T-UDC} provides a tight condition for universal non-uniform recovery by BP. 
	When the columns of $\A_0$ are linearly independent, the T-UDC is equivalent to a sparse signal recovery condition called the \emph{null space property} (\cite{Cohen:JAMS09}, see also \cite[Definition 4.1.]{Foucart:13}). 
	In particular, a matrix $\A=[\A_0, \A_-]$ is said to satisfy the null space property if for any $\0 \ne \v = (\v_0, \v_-)$ in the null space of $\A$, it holds that
	\begin{equation}\label{eq:nsp}
		\|\v_0\|_1 < \|\v_-\|_1. 
	\end{equation}
	It is straightforward to verify that \cref{eq:nsp} is equivalent to \cref{eq:equivalent-TUDC-null} by using the assumption that $\A_0$ has linearly independent columns. 
	Therefore, it follows from \cref{thm:T-UDC-equivalent} that \cref{eq:nsp} is equivalent to the T-UDC. 
	
	\cref{thm:iur} \ref{itm:iur-UDC} shows that the UDC (equivalently, the URC) is a sufficient condition for universal non-uniform recovery by BP and OMP. 	
	When the columns of $\A_0$ are linearly independent, we claim that the UDC is equivalent to the \emph{exact recovery condition} in
	sparse signal recovery condition \cite[Theorem A]{Tropp:TIT04} given by
	\begin{equation}\label{eq:exact-recovery-condition}
	\forall \w \in \cA_-
	\ \ \text{it holds that} \ \
	\|\A_0^\dag  \w\|_1 < 1,
	\end{equation}
	where $\A_0^\dag$ denotes the pseudo-inverse of $\A_0$. 
	To see that it is equivalent to the UDC, we use \cref{thm:sparse-recovery-dualpoint} and the relation $\A_0^\dag = (\A_0^\transpose \A_0)^{-1}\A_0^\transpose$ to conclude that
	\begin{equation}
	\begin{split}
	\text{UDC holds} \iff &~\forall \v \in \cD(\A_0) 
	\ \ \text{it holds that} \ \
	\|\A_-^\transpose \v\|_\infty < 1\\
	\iff &~\forall \w \in \cA_-, ~\forall \u \in U_{d_0}
	\ \ \text{it holds that} \ \
	|\langle \w, (\A_0^\dag)^\transpose  \u \rangle| < 1 \\
	\iff &~\forall \w \in \cA_-, ~\forall \u \in U_{d_0}
	\ \ \text{it holds that} \ \
	|\langle \u, \A_0^\dag  \w \rangle| < 1 \\
	\iff &~\forall \w \in \cA_-
	\ \ \text{it holds that} \ \
	\|\A_0^\dag  \w\|_1 < 1,
	\end{split}
	\end{equation}
	which is \cref{eq:exact-recovery-condition}. Consequently,  
	\cref{thm:iur} \ref{itm:iur-UDC} is not new for sparse signal recovery. 
	However, it does show that the UDC is a generalization of the exact recovery condition. 
	
	\cref{thm:iur} \ref{itm:iur-UDC} also provides new geometrically interpretable conditions for sparse signal recovery. 
	It states that universal non-uniform recovery is guaranteed if the points in $\cA_0$ are sufficiently well distributed as measured by the inradius $r_0$ and the points in $\cA_-$ are sufficiently well separated from (i) all points in $\cD(\A_0)$ for G-UDC, or (ii) all points in $\spann(\A_0)$ for G-USC.
	As for the instance non-uniform sparse signal recovery setting, one may check whether the G-UDC and G-USC are satisfied for a given $\A$ since all of the quantities in these conditions are easy to compute. 
	
	\subsection{Conditions for uniform recovery}
	
	To derive conditions that guarantee any $d_0$-sparse vector $\c$ can be recovered by BP and OMP regardless of the support of $\c$, we may require that the universal non-uniform recovery conditions in \cref{thm:iur} are satisfied for all possible partitions of the data $\cA$ into $\cA_0 \cup \cA_-$, where $\cA_0$ contains $d_0$ atoms. 
	This leads to the following theorem.
	
	
	\begin{theorem} \label{thm:ur}
		Given a dictionary $\cA$, any $d_0$-sparse vector $\c$ is the unique solution in $\BP(\A, \b)$ and $\OMP(\A, \b)$ where $\b = \A \c$ if for any partition $\cA=\cA_0\cup\cA_-$ where $\cA_0$ contains $d_0$ atoms, the columns of $\A_0$ are linearly independent and either the UDC in \cref{eq:UDC} (equivalently, the URC in \cref{eq:URC}) holds, or the G-UDC in \cref{eq:G-UDC} holds, or the G-USC in \cref{eq:G-USC} holds.
	\end{theorem}
	
	Since we proved the UDC is equivalent to the exact recovery condition~\cref{eq:exact-recovery-condition} when the columns of $\A_0$ are linearly independent, the requirement in \cref{thm:ur} that the UDC holds is equivalent to requiring that the exact recovery condition holds. 
	
	Let us now compare the results in \cref{thm:ur} with existing results in sparse signal recovery. To this end, the \emph{mutual coherence} of $\cA$ is defined as 
	\begin{equation}
	\mu(\cA) := \max_{\a_i \in \cA, \a_j \in \cA, i \ne j} |\langle \a_i, \a_j \rangle|,
	\end{equation}
	which should not be confused with the \emph{coherence} in~\cref{eq:def-coherence}.  It is known~\cite{DonohoElad:PNAS03,Tropp:TIT04} that if
	\begin{equation}\label{eq:mu-condition}
	\mu(\cA) < \tfrac{1}{2{d_0}-1},
	\end{equation}
	then both BP and OMP correctly recover any $d_0$-sparse vector. The following result shows how the mutual coherence condition~\cref{eq:mu-condition} is related to \cref{thm:ur}.

	
	
	\begin{theorem}
		If the mutual coherence condition  \cref{eq:mu-condition} holds, then 
		the columns of $\A_0$ are linearly independent and the \text{G-USC} \cref{eq:G-USC} holds, which implies using \cref{thm:G-USC} that the \text{G-UDC} 
		holds, which implies using \cref{thm:G-UDC} that the \text{UDC} 
		holds.\label{thm:sparse-recovery-compare}
	\end{theorem}
	\begin{proof}
		The fact that the columns of $\A_0$ are linearly independent when the mutual coherence condition holds is well known in sparse recovery~\cite{DonohoElad:PNAS03}.
		Therefore, we only need to show that the G-USC holds. 
		We first derive a bound on $r_0$. From \cref{thm:sparse-recovery-dualpoint}, any nonzero $\v \in \cK_0^o$ can be written as $\v = \A_0 (\A_0^{\transpose} \A_0) ^{-1} \u$ for some $0 < \|\u\|_\infty \le 1$. Letting $\lambda_{\min}(\A_0^{\transpose} \A_0)$ denote the minimum eigenvalue of $\A_0^{\transpose} \A_0$, it follows that
		\begin{align}\label{v-bound}
			\|\v\|_2 ^2 
			&= \u^\transpose (\A_0^{\transpose} \A_0)^{-1} \u 
			\leq \frac{\|\u\|_2^2}{\lambda_{\min}(\A_0^{\transpose}\A_0)}
			\leq \frac{d_0 \cdot \|\u\|_\infty^2}{\lambda_{\min}(\A_0^{\transpose}\A_0)}
			\leq \frac{d_0}{\lambda_{\min}(\A_0^{\transpose}\A_0)}.
		\end{align}
		Since the diagonal entries of $\A_0^{\transpose} \A_0$ are $1$ and the magnitude of the off-diagonal entries is bounded by $\mu(\cA)$, it follows from Gershgorin's disc theorem that $\lambda_{\min} (\A_0^{\transpose} \A_0) \ge 1 - ({d_0}-1)\mu(\cA)$, which then combined with~\cref{v-bound} gives 
		$\|\v\|_2 ^2 \le \frac{{d_0}}{1 - ({d_0}-1)\mu(\cA)}$ for all $\v\in\cK_0^o$.
		Consequently, \cref{thm:inradius-coveringradius-circumradius} implies that
		\begin{equation}\label{r0-bound}
		r_0 = 1/R_0 \ge \sqrt{1 - ({d_0}-1)\mu(\cA)} / \sqrt{d_0}.
		\end{equation}
		
		We now give an upper bound on the right-hand side of the G-USC. By definition,
		\begin{equation}\label{mu.span}
		\mu(\spann(\cA_0), \cA_-) 
		\equiv \max\{\|\A_-^\transpose \v\|_\infty : \v \in \spann(\cA_0) \ \text{and} \ \|\v\|_2 = 1\}.
		\end{equation}
		To bound $\|\A_-^\transpose \v\|_\infty$ for each $\v \in \spann(\cA_0)$ with $\|\v\|_2 = 1$, consider the problem
		\begin{equation}
		\c^* = \arg\min_{\c} \ \|\c\|_1 \st \v = \A_0 \c,
		\end{equation}
		and its dual problem
		\[
		\w^* = \arg\max_\w \ \langle \w, \v \rangle \st \|\A_0^{\transpose} \w\|_{\infty} \le 1.
		\]
		Since strong duality holds for linear problems and the primal problem is feasible, we known that $\|\c^*\|_1 = \langle \w^*, \v \rangle$.
		Moreover, we can decompose $\w^*$ into two orthogonal components given by $\w^* = \w^\perp + \w ^\parallel$, where $\w ^\parallel \in \spann(\cA_0)$. 
		Since $\v$ and all columns of $\A_0$ lie in $\spann(\cA_0)$, we have $\langle \w^*, \v \rangle = \langle \w^\parallel, \v \rangle$ and $\|\A_0^{\transpose} \w^\parallel\|_\infty = \|\A_0^{\transpose} \w^*\|_\infty \le 1$. 
		This establishes that $\w^\parallel \in \cK_0^o$, and therefore
		$\|\w^\parallel\|_2 \le R_0 = 1 /r_0$.
		It follows that
		\begin{equation}
		\|\c^*\|_1 = \langle \w^*, \v \rangle = \langle \w^\parallel, \v \rangle \le \|\w^\parallel\|_2 \|\v\|_2 \le 1/ r_0,
		\end{equation}
		which may then be used to show that
		\begin{equation}
		\|\A_-^\transpose \v\|_\infty = \|\A_-^\transpose \A_0 \c^*\|_\infty \le \|\A_-^\transpose \A_0\|_\infty \|\c^*\|_1
		\le \mu(\cA) / r_0.
		\end{equation}
		Combining this inequality with~\cref{mu.span}, \cref{r0-bound}, and \cref{eq:mu-condition} yields
		\begin{equation}
		\begin{split}
		\mu(\spann(\cA_0), \cA_-) &\le \mu(\cA) /r_0 = r_0  (\mu(\cA) /r_0^2)\\
		&\le r_0 \frac{{d_0} \mu(\cA)}{1 - ({d_0}-1)\mu(\cA)} = r_0  + r_0  \frac{\mu(\cA)(2{d_0}-1) - 1}{1-({d_0}-1)\mu(\cA)} < r_0.
		\end{split}
		\end{equation}
		This proves that the G-USC holds, and completes the proof. 
	\end{proof}


	\section{Conclusions and Future Work}
	\label{sec:conclusions}
	
	In this work, we studied the BP and OMP algorithms as a tool for the task of subspace-preserving recovery. 
	Our key results are sufficient conditions for both instance and universal subspace-preserving recovery characterized by the inradius and coherence properties of the data, which have clear geometric interpretations. 
	We further showed that our results apply to the traditional sparse signal recovery problem and that some of our conditions may be regarded as generalization of well-known conditions in sparse recovery. 
	We believe that these results provide new perspectives into the traditional sparse recovery problem.
	
	The analysis in this paper assumes the data are not corrupted by noise. As a follow-up task, it is important to understand whether BP and OMP are able to recover subspace-preserving solutions for noisy data. Prior work on subspace clustering has provided partial results for noisy data \cite{Soltanolkotabi:AS14,Wang:JMLR16,Tschannen:TIT18}, indicating that our conditions for subspace-preserving recovery may be generalizable to the case of noisy data as well. 
	
	\appendix
	\section{Proofs} 
	
	\subsection{A lemma for establishing the definition of $\cK_0^o$}
	
	\begin{lemma}\label{thm:polar-set-equivalent}
		For an arbitrary set $\cA_0 \subseteq \cS_0 \subseteq \Re^D$, we have
		\begin{equation}
		\{\v \in \cS_0: |\langle \v, \a \rangle|\le 1, \forall \a \in \cK_0\} = \{\v \in \cS_0: |\langle \v, \a \rangle|\le 1, \forall \a \in \pm\cA_0\}.
		\end{equation}
	\end{lemma}
	\begin{proof}
		Since $\pm\cA_0 \subseteq \cK_0$, it follows that $\{\v \in \cS_0: |\langle \v, \a \rangle|\le 1, \forall \a \in \cK_0\} \subseteq \{\v \in \cS_0: |\langle \v, \a \rangle|\le 1, \forall \a \in \pm\cA_0\}$.
		To prove the opposite inclusion, we take any $\bar{\v} \in \cS_0$ that satisfies $|\langle \bar{\v}, \a \rangle|\le 1$ for all $\a \in \pm\cA_0$, and any $\bar{\a} \in \cK_0$, and show that $|\langle \bar{\v}, \bar{\a} \rangle|\le 1$. 
		
		Since $\cK_0 = \conv(\pm \cA_0)$, there exist $\{c_j^+ \ge 0\}_{j: \a_j \in \cA_0}$ and $\{c_j^- \ge 0\}_{j: \a_j \in \cA_0}$ such that $\sum_{j: \a_j \in \cA_0} (c_j^+ + c_j^-) = 1$ and $\bar{\a} = \sum_{j: \a_j \in \cA_0} (c_j^+ \a_j - c_j^- \a_j)$. It follows that
		\begin{equation*}
			\begin{split}
				|\langle \bar{\v}, \bar{\a} \rangle| &= |\langle \bar{\v}, \sum_{j: \a_j \in \cA_0} (c_j^+ \a_j - c_j^- \a_j) \rangle|
				\le \sum_{j: \a_j \in \cA_0} |c_j^+ - c_j^-|  |\langle \bar{\v},  \a_j \rangle| \\
				&\le \sum_{j: \a_j \in \cA_0} |c_j^+ - c_j^-|
				\leq \sum_{j: \a_j \in \cA_0} (c_j^+ + c_j^-)
				= 1,
			\end{split}
		\end{equation*}
		which completes the proof.
	\end{proof}

	\subsection{Proof of \cref{thm:inradius-coveringradius-circumradius} (relationship between $r_0$, $R_0$ and $\gamma_0$)}
	\begin{proof}
		From \cite{Soltanolkotabi:AS12} we know that $r_0 = 1 / R_0$. Thus, it remains to show that $R_0 = 1 / \cos(\gamma_0)$. 
		From the definition of the covering radius $\gamma_0$ 
		we have
		\begin{equation}
		\gamma_0 = \sup_{\w \in \cS_0 \cap \Sp^{D-1}}\theta(\pm \cA_0, \w)
		=\sup_{\w \in \cS_0 \cap \Sp^{D-1}}\inf_{\v \in \pm\cA_0}\theta(\v, \w).
		\label{eq:prf-gamma0}
		\end{equation}
		By taking the cosine on both sides of \cref{eq:prf-gamma0}, we get
		\begin{equation}
		\cos(\gamma_0) 
		= \min_{\w \in \cS_0 \cap \Sp^{D-1}}\sup_{\v \in \pm\cA_0} \cos(\theta(\v, \w)) 
		= \min_{\w \in \cS_0 \cap \Sp^{D-1}} \|\A_0^{\transpose} \w\|_\infty.
		\label{eq:prf-cos-gamma0}
		\end{equation}
		On the other hand, by the definition of circumradius, we have
		\begin{equation}
		R_0
		= \max\{ \|\v\|_2 : \|\A_0^{\transpose} \v\|_\infty \le 1 \ \text{and} \ \v \in\cS_0\}.
		\label{eq:prf-circumradius}
		\end{equation}
		To complete the proof we show that the RHS of \cref{eq:prf-circumradius}  is the reciprocal of \cref{eq:prf-cos-gamma0}, i.e., 
		\begin{equation}
		\max\{ \|\v\|_2 : \|\A_0^{\transpose} \v\|_\infty \le 1 \ \text{and} \ \v \in\cS_0\}
		= 1/\min_{\w \in \cS_0 \cap \Sp^{D-1}} \|\A_0^{\transpose} \w\|_\infty.
		\label{eq:prf-equi-gamma-K0}
		\end{equation}	
		To see that \cref{eq:prf-equi-gamma-K0} holds, let $\v^*$ and $\w^*$ be, respectively, an arbitrary solution to the optimization problems on the LHS and RHS of  \cref{eq:prf-equi-gamma-K0}. 
		Defining $\bar{\v} = \w^* / \|\A_0^{\transpose} \w^*\|_\infty$, we get $1/\|\A_0^{\transpose} \w^*\|_\infty = \|\w^*\|_2 / \|\A_0^{\transpose} \w^*\|_\infty = \|\bar{\v}\|_2 \le \|\v^*\|_2$, where the inequality follows because $\bar{\v}$ satisfies the constraints of the optimization on the LHS of \cref{eq:prf-equi-gamma-K0}, i.e., $\|\A_0^{\transpose} \bar{\v}\|_\infty \le 1$ and $\bar{\v} \in \cS_0$. 
		On the other hand, we may define $\bar{\w} = \v^* / \|\v^*\|_2$ and get $ 1 / \|\v^*\|_2 \ge \|\A_0^{\transpose} \v^*\|_\infty / \|\v^*\|_2 = \|\A_0^{\transpose} \bar{\w}\|_\infty \ge \|\A_0^{\transpose} \w^*\|_\infty$, where the final inequality follows from the fact that $\bar{\w}$ satisfies the constraint of the optimization on the RHS of \cref{eq:prf-equi-gamma-K0}.
		Combining these two gives $1 / \|\A_0^{\transpose} \w^*\|_\infty = \|\v^*\|_2$ so that \cref{eq:prf-equi-gamma-K0} holds. 
	\end{proof}
	
	\section*{Acknowledgments}
	We would like to acknowledge that the equivalent form of the T-UDC that is given in \cref{eq:equivalent-TUDC-null} is established by Dr. Mustafa Devrim Kaba at Johns Hopkins University. 
	
	%
	%

	\bibliographystyle{siamplain}
	
	\bibliography{biblio/vidal,biblio/vision,biblio/math,biblio/learning,biblio/sparse,biblio/geometry,biblio/dti,biblio/recognition,biblio/surgery,biblio/coding,biblio/segmentation,biblio/dataset}
\end{document}